\documentclass{article}

\usepackage[utf8]{inputenc} 
\usepackage[T1]{fontenc}    
\usepackage{lmodern}
\usepackage{hyperref}       
\usepackage{url}            
\usepackage{booktabs}       
\usepackage{amsfonts}       
\usepackage{nicefrac}       
\usepackage{microtype}      
\usepackage{xcolor}         

\usepackage{float}			

\usepackage[margin=1in]{geometry}

\usepackage{amsthm}
\usepackage{amsmath}
\usepackage{amssymb}
\usepackage{graphicx}

\newcommand{\relu}{\text{ReLU}}
\newcommand{\R}{\mathbb{R}}
\newcommand{\N}{\mathbb{N}}

\newcommand{\CF}{\mathcal{C}(F)}

\newcommand{\sgn}{\text{sgn}}

\newtheorem{theorem}{Theorem}
\newtheorem{lemma}[theorem]{Lemma}
\newtheorem{defn}[theorem]{Definition}

\title{Algorithmic Determination of the Combinatorial Structure of the Linear Regions of ReLU Neural Networks}

\author{%
	Marissa A. Masden\\
	Department of Mathematics\\
	University of Oregon\\
	Eugene, OR 97403\\
	\texttt{mmasden@uoregon.edu} \\
}

\begin{document}

\maketitle

\begin{abstract}

We algorithmically determine the regions and facets of all dimensions of the canonical polyhedral complex, the universal object into which a ReLU network decomposes its input space.  We show that  the locations of the vertices of the canonical polyhedral complex along with their signs with respect to layer maps determine the full facet structure across all dimensions.We present an algorithm which calculates this full combinatorial structure, making use of our theorems that the dual complex to the canonical polyhedral complex is cubical and it possesses a multiplication compatible with its facet structure. The resulting algorithm is numerically stable, polynomial time in the number of intermediate neurons, and obtains accurate information across all dimensions. This permits us to obtain, for example, the true topology of the decision boundaries of networks with low-dimensional inputs.  We run empirics on such networks at initialization, finding that width alone does not increase observed topology, but width in the presence of depth does.  Source code for our algorithms is accessible online at \hyperlink{https://github.com/mmasden/canonicalpoly}{https://github.com/mmasden/canonicalpoly}. 
\end{abstract}

\section{Introduction}

 For fully-connected ReLU networks, the canonical polyhedral complex of the network, as defined in \cite{gloriginal}, encodes its decomposition of input space into linear regions and determines key structures such as the decision boundary. Investigation of properties and characterizations of this decomposition of input space are ongoing, in particular with respect to counting the top-dimensional linear regions \cite{BoundingAC,haninactivation, montufar}, since these bounds give one measure of the expressivity of the associated network architecture. However, understanding of adjacency of regions and more generally the connectivity aspects of lower-dimensional facets are to our knowledge generally undocumented. The connectivity properties across dimensions are necessary to relate combinatorial properties of the polyhedral complex of a network to, for example, the topology of regions into which the decision boundary partitions input space, geometric measurements such as local curvature, or other notions of geometric and topological expressivity, as explored in \cite{guss,Bianchini2014}. 

It is common to describe the top dimensional regions of the input space using "activation patterns" or "neural codes" recorded as vectors in $\{0,1\}^N$ (for example, in \cite{neuralcodes}). Unfortunately,  knowing which activation patterns are present in the top-dimensional regions of a network does not determine their pairwise intersection properties (Theorem \ref{thm:topdimbad}), and computing the intersections of these regions directly is not numerically stable. Inspired by the theory of oriented matroids in hyperplane arrangements \cite{hyperplanes}, we extend the notion of "sign" to include boundary cases, possible in artificial but not biological networks, and work up in dimension from vertices rather than down from regions. 
We show that with full probability, computing the vertices present in the polyhedral complex and recording the signs of the network's activity in intermediate layers determines all face relations in the polyhedral complex. The sign sequence data can be viewed as a labeling scheme which tracks face relations, but theoretically it defines a combinatorial duality isomorphism of the polyhedral complex with a subcomplex of a hypercube (see Figure \ref{fig:canonicalpoly}).  
We provide an algorithm which produces the information necessary to obtain all combinatorial properties of the canonical polyhedral complex, together with its face relations, and as a result, substructures such as the decision boundary. This algorithm is numerically stable with respect to polyhedral intersection, and has polynomial runtime in the number of intermediate neurons at initialization.
The ability to compute the explicit decision boundary of a network provides a new means to evaluate topological expressivity of network functions. 
We  demonstrate the utility of the sign sequence complex by obtaining metrics about topological properties of decision boundaries for fully-connected networks at initialization, which indicates that as width of a network increases, there is more topological complexity as well as variability in topology for deeper networks than for shallow networks. 

In summary, our main contributions are: 
\label{sec:intro}
\begin{itemize}
	\item We prove the existence of a combinatorial description of the geometric dual of the canonical polyhedral complex of a ReLU neural network which consists of a generalization of activation patterns, which we call the \textit{sign sequence complex}. Furthermore, using a product structure which we prove to be well-defined, we show that the only information needed to determine the face poset structure of the sign sequence complex is the sign sequences of the vertices of the polyhedral complex.
	\item We implement a corresponding algorithm for obtaining the sign sequence complex which is numerically stable and runs in polynomial time in the number of intermediate neurons, and exponential time in the input dimension.
	\item  We show that the sign sequence complex can be naturally restricted to particular substructures of the polyhedral complex of a ReLU network. In particular, a chain complex describing the topology of the decision boundary can be obtained with simple operations on a subset of the vertices of the polyhedral complex, together with  their sign sequences. 
	\item We demonstrate the usefulness of this characterization of a network by obtaining the statistics of ReLU networks' decision boundaries' topological properties at initialization, dependent on architecture. These experiments provide empirics that depth of a network plays a stronger role in topological expressivity, at least at initialization, than the number of intermediate neurons.
\end{itemize} 

\section{Related work}

The seminal paper in \cite{gloriginal} establishes a high-level view of the cellular structure of the canonical polyhedral complex, but does not establish explicit lower-dimensional information. Under weak assumptions, they show the canonical polyhedral complex's $(n_0-1)$-skeleton may be described as the preimages of hyperplanes in each layer, but arbitrary $k$-skeletons are unexamined for $k<(n_0-1)$, as are general face relations. The subsequent work \cite{glmsecond} establishes local models for the polyhedral structure at the intersection of hyperplanes, but does not address deeper network structures as we do here. 

In \cite{haninregions} and \cite{haninactivation}, the preimages of hyperplanes which correspond to various dimensional subcomplexes are discussed, but primarily recording the totality of their volumes and counting the top-dimensional regions, respectively, and not obtaining their adjacency relations. In particular, in these works properties of hyperplane arrangements are used to establish \textit{statistical} properties of the canonical polyhedral complex.  While our work does rely on properties of hyperplane arrangements in a similar way, we focus on encoding the \textit{topological} full face poset. In addition, others who approach explicit computation of linear regions as in \cite{Zhang2020Empirical} do so using $H$-representations of polyhedra, and we use $V$-representations. While in theory one could intersect top-dimensional regions pairwise to obtain topological information such as whether two polyhedra share a low-dimensional face, this is not numerically stable, especially when the linear equations involved arise from matrix multiplication. Our "working forward" method affords a priori knowledge of equalities, avoiding issues potentially arising from numerical error in polyhedral intersection. 

Other characterizations of the combinatorics of ReLU networks' polyhedral complexes exist, but lack explicit implementation or applicability to deeper networks. In \cite{powerdiagram}, the regions of the canonical polyhedral complex are described according to the roots of a polynomial, but no algorithm is presented on how to obtain these roots, nor how to explicitly determine whether two polyhedra are connected by a shared face. In \cite{tropical}, a tropical characterization of the polyhedral complex including its face relations relies on the translation of network functions to tropical rational functions with integer coefficients, and its application in \cite{tropicaldb} appears limited to networks with single hidden layers. In contrast, we believe that the vertices present in the sign sequence cubical complex are stable in open sets of parameter space and do not change through the network training except at single thresholds. In \cite{neuralcodes} a characterization of the regions of single-layer hyperplane networks is established which relies on similar sign labelings, but the methods do not apply to deeper networks. Our theoretical work is applicable to networks with additional hidden layers. Furthermore, in contrast to biological papers such as \cite{curto} where boundary structure is not clearly defined, in the context of artificial networks the boundary intersections are in fact computable. 

Focusing on implementation, there is no related work which we know to cite. Our code is the first publicly available which obtains the full cellular poset structure of the polyhedral complex and gives precise topological calculations of the decision boundary.

\section{The sign sequence cubical complex}

We define the sign sequence cubical complex and justify its importance before describing an algorithm for its computation. Detailed definitions and proofs are given in the Appendix, where the reader can find them with the given numbering.

\subsection{Preliminaries}

\label{sec:def}

We work with fully-connected, feedforward ReLU networks (see Definition \ref{def:ReLU} in Appendix \ref{A:Definitions}), following the framework in \cite{gloriginal}. In this framework, if $F$ is expressed as $A_m \circ F_{m-1}\circ ... \circ F_2 \circ F_1$, where each $F_i=\relu \circ A_i$ for an affine map $A_i$, then for the networks under investigation, the last affine map $A_m$ has image in $\mathbb{R}$ and is not followed by the ReLU function. We refer to the activity of each individual hidden unit of the network as the $(i,j)$th \textit{node map}, $F_{ij}$ (Definition \ref{def:nodemap}). 

For fixed $i$, the solutions to $F_{ij}=0$ are hyperplanes in $\mathbb{R}^{n_i}$, which together form a hyperplane arrangement \cite{hyperplanes}. Such are equipped naturally with the structure of a polyhedral complex (Definition \ref{def:polycomplex}). In \cite{gloriginal} the canonical polyhedral complex $\mathcal{C}(F)$ is defined to consist precisely of the cells in $\mathcal{R}^{n_0}$ which are intersections of the preimages of one cell $R_i \subset \mathbb{R}^{n_i}$ under $F_{i-1}\circ .. \circ F_{0}$ for each $i$ (Definition \ref{def:canonicalcomplex}). It is established in \cite{gloriginal} and in \cite{grunert} that each intermediate complex $\mathcal{C}(F_k \circ ... \circ F_1)$ is a polyhedral complex which subdivides the previous one. An immediate result is that $F$ is affine linear on each cell of $\mathcal{C}(F)$. 

We pay particular attention to certain subsets of ReLU neural networks, called \textit{generic} (Definition \ref{def:generic}) and \textit{supertransversal} (Definition \ref{def:supertransversal}) ReLU networks, which satisfy additional conditions. These conditions guarantee that the combinatorial results in the next section hold, but are nonrestrictive in light of the following lemma. 

\newtheorem*{lem:generic}{Lemma \ref{lem:generic}}
\begin{lem:generic}
	Supertransversality is full measure in $\mathbb{R}^P$, where $P$ is the set of network parameters. Additionally, it is \textit{fiberwise generic}, that is, with fixed network weights, the set of biases on which supertransversality is generic is open and full measure. 
\end{lem:generic} 

That almost all networks are, additionally, \textit{generic}, is established in \cite{gloriginal}. This, along with Lemma \ref{lem:generic},  guarantees that in all but a measure-zero subset of networks, the theory developed below will hold.

\subsection{The sign sequence cubical complex}

It is common to use binary strings of length $N$ (which we will denote using -1 and 1) as a labeling scheme to describe which neurons are active at a point in a ReLU network's input space, e.g. in \cite{neuralcodes}. (Here, $N$ is the number of intermediate neurons in the network.) However, this is insufficient to describe the combinatorics of a network's canonical polyhedral complex. 

\newtheorem*{thm:topdimbad}{Theorem \ref{thm:topdimbad}}
\begin{thm:topdimbad}
	There exists a pair of networks $F_1$ and $F_2$ such that the set of strings in $\{-1,1\}^N$ encoding the activation patterns in the interiors of the cells of $\mathcal{C}(F_1)$ and $\mathcal{C}(F_2)$ are equal, but the polyhedral complexes $\mathcal{C}(F_1)$ and $\mathcal{C}(F_2)$ are not combinatorially equivalent, and the decision boundaries of the networks are not homotopy equivalent.
\end{thm:topdimbad}

 We propose instead to use \textit{sign sequences}, defined below, as a means to label all regions and thereby fully encode the combinatorial properties of a network's canonical polyhedral complex.

\newtheorem*{def:signsequence}{Definition \ref{def:signsequence}}
\begin{def:signsequence}
	Define $s: \mathcal{C}(F) \to \{-1,0,1\}^N$ by $s_{ij}(C) = \sgn(F_{ij}(C))$. We call $s(C)$ the \textit{sign sequence} of the cell $C$.
\end{def:signsequence}

Sign sequences are sufficent to list the cells of $\mathcal{C}(F)$. 

\newtheorem*{thm:injective}{Theorem \ref{thm:injective}}
\begin{thm:injective}
	The function $s$ is well-defined and injective.
\end{thm:injective}

Furthermore, the sign sequence of a cell is determinative under conditions of supertransversality and genericity. For example, it encodes the dimension of the cell, in that the number of zeros in the sign sequence of a cell is equal to its codimension. 

\newtheorem*{lem:countzeros}{Lemma \ref{lem:countzeros}}
\begin{lem:countzeros}
	Let $F$ be generic and supertransversal. Let $C$ be a $k$-cell of $\CF$. Then $s(C)$ has exactly $n_0-k$ entries which are zero. 
\end{lem:countzeros}

This leads to a key \textit{geometrically dual} relationship between the canonical polyhedral complex of a supertransversal network and a subcomplex of $[-1,1]^N$, sending $k$-cells in $\mathcal{C}(F)$ to $(n_0-k)$-faces in a cube. We will call this subcomplex $\mathcal{S}(F)$. Recall that a pure (sub)complex is a complex where every facet is contained in some polytope of uniform top dimension.


\newtheorem*{thm:cubicalcomplex}{Theorem \ref{thm:cubicalcomplex}}
\begin{thm:cubicalcomplex}
	For each generic, supertransversal neural network $F$, the image of the map $S: \CF \to \{-1,0,1\}^N$ gives a geometric duality between $\mathcal{C}(F)$ and a pure subcomplex $\mathcal{S}(F)$ of the hypercube $[-1,1]^N$ endowed with the product CW structure.  
\end{thm:cubicalcomplex} 

\begin{figure}
	\centering 
	
	\includegraphics[width=0.4\linewidth]{"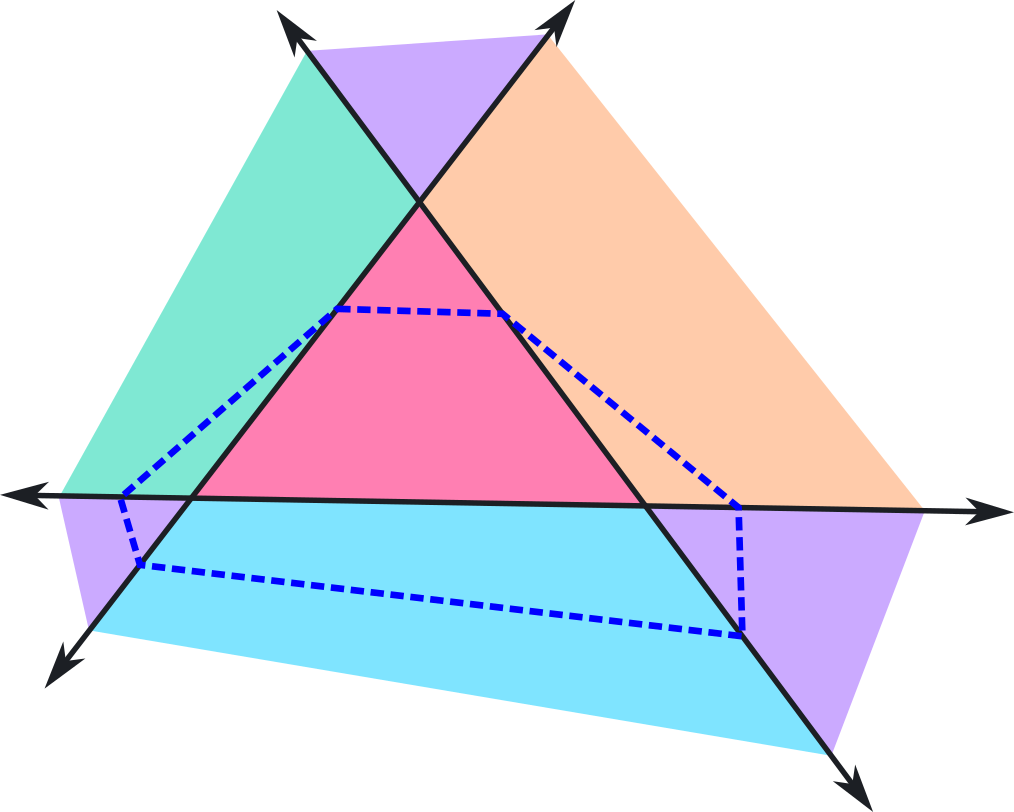"} \hspace{0.1\linewidth}\includegraphics[width=0.4\linewidth]{"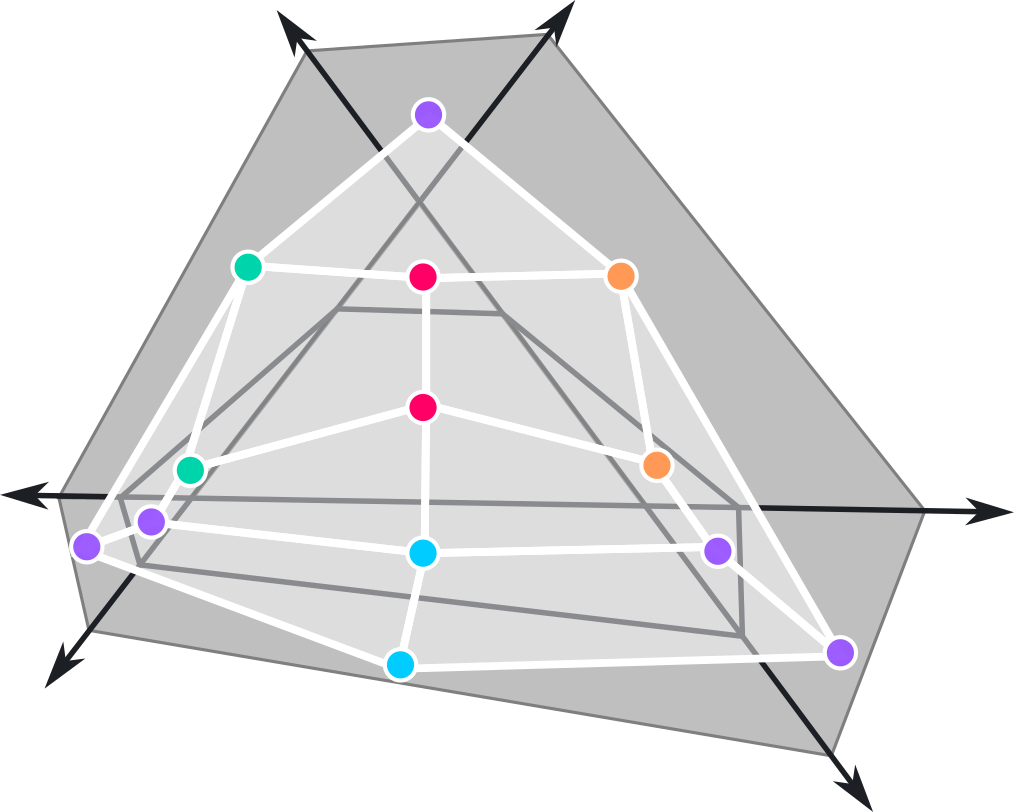"}

	\caption{An illustration of a canonical polyhedral complex, $\mathcal{C}(F)$, on the left. Its geometric dual sign sequence complex $\mathcal{S}(F)$ is superimposed in white on the right, with one vertex for each region of $\mathcal{C}(F)$. As we prove in general, it is cubical, with each two-cube (quadrilateral) containing a unique vertex of $\mathcal{C}(F)$.}
	
	\label{fig:canonicalpoly}
\end{figure}

In particular, vertices of $\mathcal{C}(F)$ correspond to $n_0$-cells of $\mathcal{S}(F)$. Since the complex $\mathcal{S}(F)$ is pure and $n_0$-dimensional, knowing which $n_0$-cells are present is sufficient to determine all face relations in the subcomplex of the hypercube. A corollary is: 

\newtheorem*{thm:verticeskey}{Corollary}
\begin{thm:verticeskey}
	The sign sequences of the vertices of $\mathcal{C}(F)$ determine the face poset of the polyhedral complex $\mathcal{C}(F)$. 	
\end{thm:verticeskey}

\subsection{Algebraic structure of the sign sequence complex}

	\begin{figure}[h]
		\centering
		
		\includegraphics[width=0.5\linewidth]{"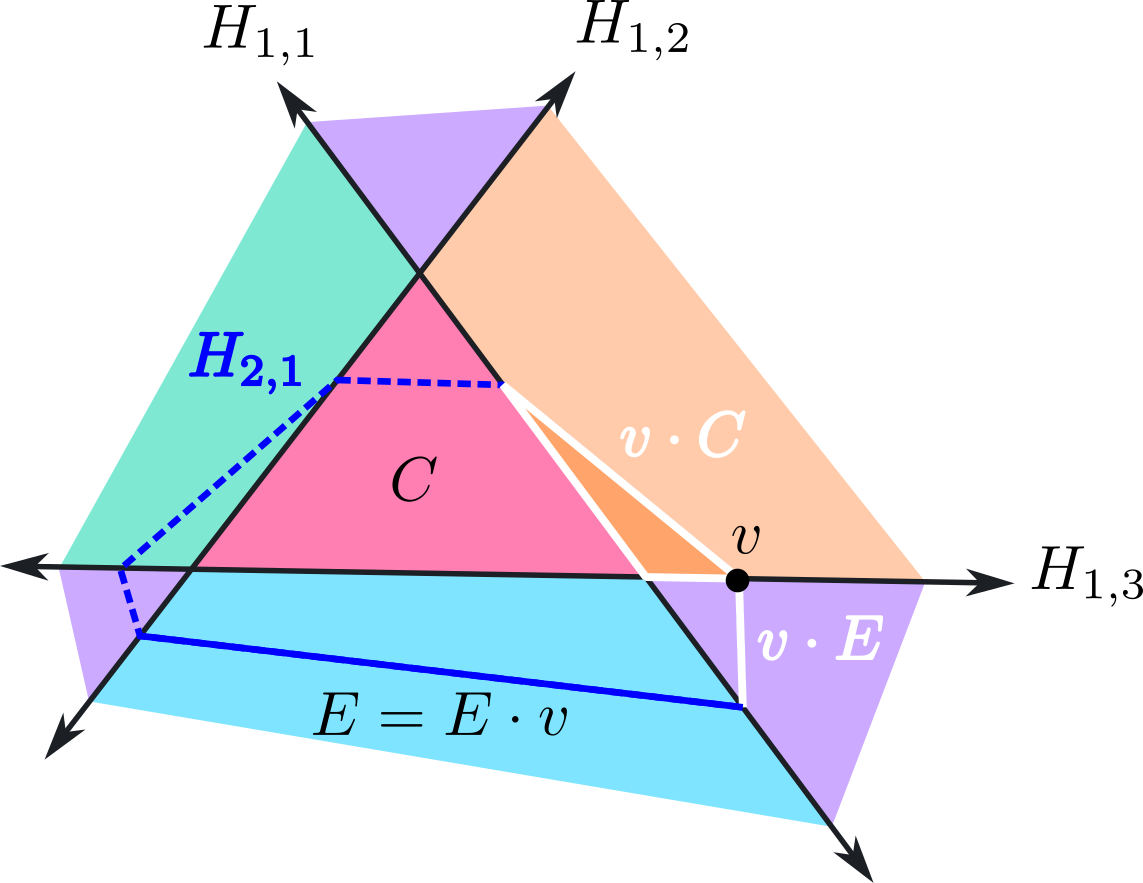"}
		
		\caption{The algebraic structure on the sign sequence cubical complex, pictured geometrically. Under one possible co-orientation of hyperplanes, these cells have the sign sequences indicted in Table \ref{tab:ss}. The product is computed and pictured for certain pairs of cells.}
		\label{fig:product}
	\end{figure}

\begin{table}
	\centering
	\begin{tabular}{cc}
		\toprule
		Cell & Sign Sequence \\ 
		\midrule
		$v$ & \texttt{(1,1, 0, 0)}\\
		$E$ & \texttt{(1,1,-1, 0)} \\
		$C$ & \texttt{(1,1, 1,-1)}	\\	
		$v \cdot C$ & \texttt{(1,1, 1,-1)}\\ 
		$v \cdot E$ & \texttt{(1,1,-1, 0)} \\ \bottomrule
	\end{tabular}
\caption{The sign sequence of the cells in Figure \ref{fig:product}, together with some computed products.}
\label{tab:ss}
\end{table}

The sign sequence complex has only one combinatorial type, cubes. This uniformity is related to a package of formal properties which will be crucial in algorithmic implementation, with the existence of a multiplicative structure being particularly helpful. 

\newtheorem*{faceproperties}{Lemma \ref{lem:faceproduct}}
\begin{faceproperties}
	If $C$ and $D$ are two cells of $\mathcal{C}(F)$, the product $S(C)\cdot S(D)$ given by: 
	
	$$(S(C)\cdot S(D))_{ij} =\begin{cases}
	S(C)_{ij} & \text{if }S(C)_{ij} \neq 0 \\ 
	S(D)_{ij} & \text{otherwise}
	\end{cases} $$ 
	
	is well-defined as a product between sign sequences; that is, there exists a cell $E$ in $\mathcal{C}(F)$ such that $S(C)\cdot S(D) = S(E)$ for all cells $C$ and $D$. Thus, sign sequences of $\mathcal{C}(F)$ form a semigroup.	
\end{faceproperties}

Following from similar constructions in hyperplane arrangements \cite{hyperplanes} and oriented matroids we obtain the following:

\newtheorem*{morefaceproperties}{Lemma \ref{lem:faceproperties}}
\begin{morefaceproperties}
	For all supertransversal networks, the following relations hold for all $C$ and $D$ in $\mathcal{C}(F)$,  where $\leq$ is the relation "is a face of":
	\begin{itemize}	
		\item $C \leq D$ if and only if $S(C)\cdot S(D) = S(D)$ 
		\item $S(C) \cdot S(D) = S(D) \cdot S(C)$ if and only if there is a cell $E$ with $D\leq E$ and $C\leq E$. 	
		\item $S(C) \cdot S(D) = S(C)$ if and only if $C$ is contained in the intersection of the maximal set of bent hyperplanes which contain $D$. 
	\end{itemize}
\end{morefaceproperties}

These characterizations are primarily useful for using code to track face relations via a discrete structure. As we implement in Section \ref{sec:db} we can use these to compute the topological properties of the decision boundary of a network using sign sequences.

\section{Computation of the sign sequence complex} 


The properties of the sign sequence complex make it straightforward to compute the combinatorial properties of the polyhedral complex of a network across all dimensions upon obtaining the sign sequences of the vertices of $\mathcal{C}(F)$.
Moreover, these sign sequences follow from locating potential vertices, thus knowing locations of the $0$ entries in its sign sequence, and then evaluating the network to obtain its remaining signs.

\subsection{Obtaining the sign sequence complex}

\label{sec:algorithm}
In \cite{gloriginal}, the canonical polyhedral complex $\mathcal{C}(F)$ is defined iteratively through layers. Letting $R^{(k)}$ be the polyhedral complex associated with the hyperplane arrangement in layer $k$, the complex $\mathcal{C}(F_k \circ ... \circ F_1)$ is given precisely by the intersection complex of $\mathcal{C}(F_{k-1}\circ ... \circ F_1)$ and $(F_{k-1}\circ ... \circ F_1)^{-1}(R^{(k)})$. (See Definition \ref{def:canonicalcomplex}). To obtain the vertices of a particular network's canonical polyhedral complex, we may therefore begin by obtaining the vertices corresponding to $\mathcal{C}(F_1)$, its first layer's canonical polyhedral complex, which is immediate.

\newtheorem*{lem:firstlayervertices}{Lemma \ref{lem:firstlayervertices}}
\begin{lem:firstlayervertices} 
	The $0$-cells of $F_1$ are given by the solutions to $$\{W_{\alpha} x = b_\alpha: \alpha \subset [n_1] \;\& \; |\alpha|=n_0  \}$$ where $W$ is the weight matrix of the network and $\alpha$ denotes a subset of the $n_1$ vertices. 
	
	A vertex $v$ obtained by solving $W_\alpha x = b_\alpha$ satisfies $s_{i}(v)=0$ iff $i \in \alpha$.  
\end{lem:firstlayervertices}

The sign sequences of the top-dimensional regions which are present in $\mathcal{C}(F_{1})$ are determined by the sign sequences of the vertices of $\mathcal{C}(F_1)$, ignoring signs corresponding to neurons in later layers, as described by Lemma \ref{lem:coboundary}. 

Following the computation of the first layer, subsequent layers' vertices may be found by analyzing the preimage of each bent hyperplane for additional intersections of $k$ bent hyperplanes from the new layer together with $n_0 - k$ bent hyperplanes from the previous layers. Since $F_{k-1}\circ ... \circ F_1$ is affine on each region of $\mathcal{C}(F_{k-1}\circ ... \circ F_1)$ we denote the affine map $A_R$. Restricted to this region, $F_{ij}(x)$ is affine, and we call this affine map $A_{ij}$.

\newtheorem*{lem:laterlayervertices}{Lemma \ref{lem:laterlayervertices}}
\begin{lem:laterlayervertices}
	Let $F$ be a generic, supertransversal neural network. 

If $C$ is a cell of $\mathcal{C}(F_{k-1} \circ ... \circ F_1)$, then $F_{ij}(C)$ is affine for all $i \leq k$. Then, 

\begin{enumerate}
	\item	All $0$-cells of $\mathcal{C}(F_k \circ ... \circ F_1)$  which are contained in the closure of $C$ and which are not already in $\mathcal{C}(F_{k-1}\circ ... \circ F_1)$ are the solution to a system of $n_0$ affine equations, of which $1 \leq \ell \leq n_0$ are of the form:
	
	$$ A_{kj}(x)=0  $$
	
	and $0 \leq n_0 - \ell \leq n_0 - 1$ equations are of the form: 
	
	$$A_{ij}(x) = 0 ; i < k $$ 
	
	Here the the $A_{ij}$ of the $n_0-\ell $ equations from earlier layers are selected such that there exists a vertex of $C$ in the intersection of the corresponding bent hyperplanes. 
	
	\item 	Furthermore, a solution $x$ to the system of equations described in (1) corresponds to a $0$-cell of $\mathcal{C}(F_k\circ ... \circ F_1)$ contained in the closure of $C$ if and only if, for all remaining $(i,j)$ pairs with $i \leq k-1$, we have that $s_{ij}(x)=s_{ij}(C)$.
	
\end{enumerate}
\end{lem:laterlayervertices}

We may therefore proceed iteratively through layers in order to obtain the full polyhedral complex. In summary, 

\newtheorem*{sum:algorithm}{Computing Sign Sequences}
\begin{sum:algorithm}
	\label{sum:algorithm}
	To obtain the vertices of $\mathcal{C}(F)$: 
	\begin{enumerate}
		\item Compute the intersections of the hyperplanes from the first layer, as in Lemma \ref{lem:firstlayervertices}. Obtain their sign sequences by evaluating $F_{ij}$ on each intersection. This obtains $\mathcal{C}(F_1)$.  
		\item To compute $\mathcal{C}(F_i)$, loop through regions $C$ in $\mathcal{C}(F_{i-1})$. On each region $C$,
		\begin{enumerate}
			\item For $1 \leq k \leq n_0$, compute the intersections of $k$ bent hyperplanes from the new layer with $n_0 - k$ bent hyperplanes from previous layers, the latter of which are selected so that their intersection forms an $n_0-k$-face of $C$. 
			\item Evaluate $F_{ij}(x)$ for each computed intersection $x$. Then keep $x$ as a vertex of $\mathcal{C}(F)$ if and only if $F_{ij}(x) = F_{ij}(C)$ for $i \leq {k-1}$, following Lemma \ref{lem:laterlayervertices}.
		\end{enumerate}
	\end{enumerate}
\end{sum:algorithm}

\subsection{Obtaining decision boundary topology from the sign sequence complex} 

\label{sec:db}

The characterization of $\mathcal{C}(F)$ as dual to a cubical complex permits us to define a straightforward mod-two cellular coboundary, which is the transpose of the boundary operation on the cubical complex.

\newtheorem*{lem:coboundary}{Lemma \ref{lem:coboundary}}
\begin{lem:coboundary}
	Let $C$ be a cell of $\mathcal{C}(F)$. Then the cells $D$ of which $C$ is a facet are given by the set of cells with sign sequence given by $S(D)_{ij}=S(C)_{ij}$ for all $i$ and $j$ except for exactly one, a location for which $S(C)_{ij}=0$.
\end{lem:coboundary}

We recover the decision boundary of a network by noting that the subcomplex of cells $C$ in $\mathcal{C}(F)$ satisfying $S(C)_{N}=0$ are those cells which satisfy $F(C)=0$. The cells satisfying $F(C)=0$ have vertices with $F(v)=0$, so by locating those vertices in $\mathcal{S}(F)$ which have a $0$ in the last coordinate, this coboundary operation may be used to construct a cochain complex of cells of the decision boundary by restricting the image of the coboundary to those cells whose last sign in their sign sequence is also zero. Mod two, the cells of $\mathcal{C}(F)$ have a boundary exactly dual to this coboundary in $\mathcal{S}(F)$.

In general, the presence of unbounded cells makes this map not quite correspond to a cellular chain complex. In particular, not every edge has two vertices. By adding a single `vertex at infinity' to unbounded edges, the corresponding cochain complex has a straightforward interpretation as the cochain complex of the one-point compactification of the decision boundary.  

\begin{figure}
	\centering
	\includegraphics[width=0.5\linewidth]{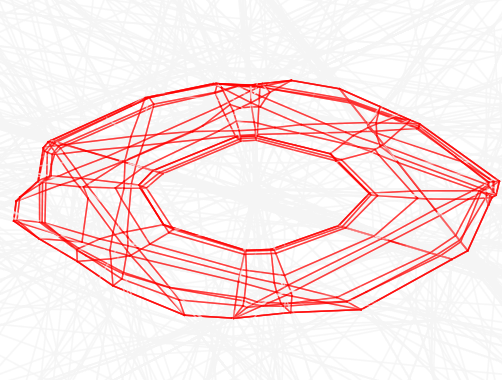}
	\caption{In red, the decision boundary of a neural network with architecture $(3,15,8,1)$. The network was trained using stochastic gradient descent to distinguish a points sampled from a torus from points sampled from an annulus at its center.}
\end{figure}

\subsection{Numerical stability and algorithmic complexity}

Naively, if we compute a solution $x$ to $F_{ij}(x)=0$, and then numerically evaluate the node map $F_{ij}(x)$, the result may not be exactly zero due to floating point error. However, a consequence of Lemma \ref{lem:laterlayervertices} is that machine epsilon-level errors obtained when solving for the location of a vertex as the intersection of $n_0$ bent hyperplanes do not lead to errors in computing the sign sequence of a vertex. When determining the sign sequence of a vertex, which of its signs are zero is determined by which hyperplanes were intersected, and the remaining signs are stable to small perturbations since the sets $F_{ij}>0$ and $F_{ij}<0$ are open sets. As long as the error in computing solutions to linear equations is small compared to the size of the cells in the polyhedral complex, the proposed algorithm will find the correct sign sequence of each vertex, and as a result the correct combinatorics of the polyhedral complex. 

Furthermore, as deep ReLU networks only have polynomially many regions in the number of hidden units, at least at initialization \cite{haninactivation}, and the number of possible combinations of $k$ neurons from $n_i$ neurons together with $n_0-k$ neurons from $n_0 + ... + n_{k-1}$ neurons is also polynomial in the total number of hidden units, it is possible to obtain the canonical polyhedral complex $\mathcal{C}(F)$ in polynomial time in the number of hidden units.

\section{Applications to network decision boundaries at initialization}
\label{sec:experiments}
We use this theoretical framework to make experimental observations. We obtain statistics about the decision boundaries of networks, and find stark differences in the behaviors of shallow and deeper network architectures. To our knowledge, this is the first experimental determination of the \textit{exact} topology of a large collection of decision boundaries with input dimension greater than two while having more than one hidden layer.

\subsection{Experimental design}

We randomly initialize 50 networks of fully-connected architectures $(k, n, 1)$ and $(k,n,n,1)$ for $2\leq k \leq 4$ with standard normal weights and biases. We will call the networks of architecture $(k,n,1)$ "shallow" and those of architecture $(k,n,n,1)$ "deep" for the purposes of comparison in this section. We compute the canonical polyhedral complex $\mathcal{C}(F)$ for each network using an implementation of the algorithm described in Section \ref{sum:algorithm} using Pytorch linear algebra solver \cite{pytorch}. We then obtain the Betti numbers $\beta_i$ for $0\leq i \leq k-1$ of the one-point compactification of the decision boundary of the network at initialization, by constructing the boundary map determined in Lemma \ref{lem:coboundary}. The Betti numbers were obtained using a Sage implementation of general chain complexes \cite{sage}. The resulting Betti numbers provide a measure of topological complexity of the decision boundary at initialization. 

\subsection{Results}

We observe that the topology of the decision boundary for shallow networks, regardless of input dimension, remains remarkably constant over the range of dimensions investigated. (See Figure \ref{fig:homologyatinit}). In contrast, for deep networks, there is both greater variability in the distribution of the topology of the decision boundary, and increasing width appears to, at least for $n_0>2$, lead to the the topological properties of the decision boundary changing in distribution as the width $n$ increases. We conjecture that a plausible explanation is that deep networks require greater width before their network functions are in the Gaussian regime.

\begin{figure}[h]
	\centering 
	\includegraphics[width=0.9\linewidth]{"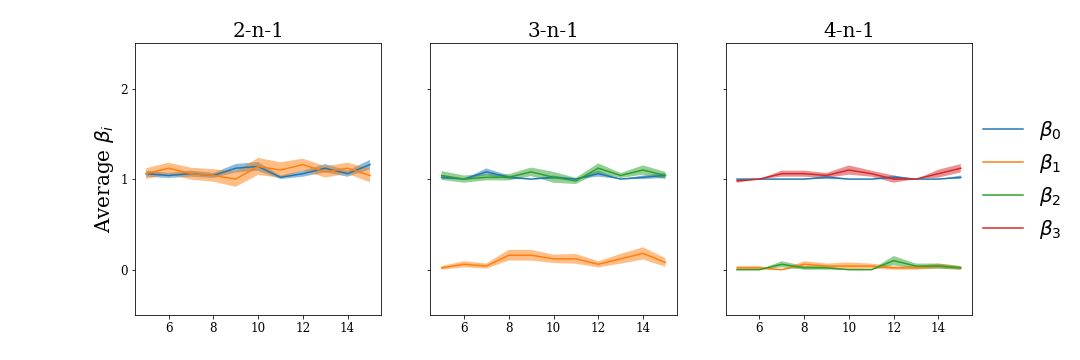"}\\
	\includegraphics[width=0.9\linewidth]{"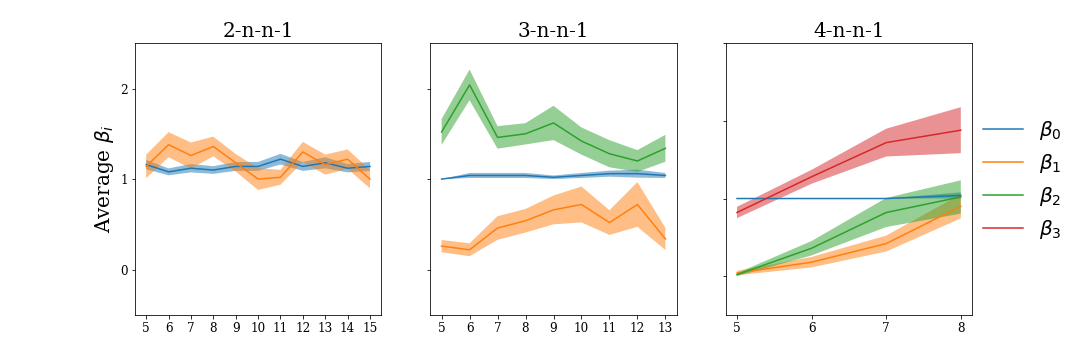"}
		\caption{Average Betti numbers of the network decision boundary at initialization for a range of networks of shallow architecture (top) with deep architecture (bottom), together with standard error. The distribution of topological properties of the decision boundaries at initialization is surprisingly constant for shallow networks, and more variable for deeper networks.}
	\label{fig:homologyatinit}
\end{figure}

The Betti numbers of the one-point compactification measures the number of bounded and unbounded components of the decision boundary. Since all unbounded components are compactified by attaching them to the same point, in the compactification they correspond to $n_0-1$ cycles belonging to the same connected component. So the number of bounded and unbounded components can be computed by $\beta_0 - 1$ and $\beta_{n_0-1} - \beta_0 + 1$, respectively. 

 We observe that \textit{bounded} connected components of the decision boundary at initialization are rare, with frequency decreasing with input dimension in both shallow and deep networks: For example, $8.1\%$ of networks of the form $(2,n,1)$ contain at least one bounded component, whereas only $0.05\%$ of networks of the form $(4,n,1)$ contain as much at initialization. Furthermore, regarding the number of unbounded components, across all shallow networks investigated, the largest number of unbounded components observed at initialization was 3, with a mode of 1 (average 1.0, 1.03 and 1.04,  for $n_0=2,3,$ and $4$ respectively). In deeper networks, in contrast, the number of unbounded components at initialization appears to be on average greater ($1.1, 1.4, 1.4$, respectively) reaching maximum observed values of $5$, $7$ and $12$ for $n_0 = 2,3,$ and $4$ respectively. However, the most common observation is still that a network at initialization has one unbounded connected component, and whether there is any trend associated with width is unclear. 
 
 Table \ref{tab:bettis} summarizes distributional information about the Betti numbers of the decision boundary, and Table \ref{tab:ccs} summarizes information about the connected components of the decision boundary. Figure \ref{fig:ccs} gives additional distributional information for selected architectures. While shallow architectures again have a very constant distribution of the number of unbounded components even across input dimension, the number of unbounded components seen at initialization in deeper architectures is much greater, and the distributional variability with width is apparent.

 \begin{table}[H]
 	\centering 
 	
 	\caption{Betti numbers of the compactified decision boundary dependent on architecture, across the range of widths studied ($50$ networks of each architecture). In $\beta_{n_0-1}$, deeper architectures exhibit greater variability and greater apparent change with width across the range of widths studied.}
 	
 	\begin{tabular}{lcccl}\toprule
 		\multicolumn{5}{c}{Shallow Architectures} \\
 		\midrule 
 		& \multicolumn{2}{c}{$\beta_0$} & \multicolumn{2}{c}{$\beta_{n_0-1}$}
 		\\\cmidrule(lr){2-3}\cmidrule(lr){4-5}
 		& Average & SE & Average & SE \\\midrule
 		$(2,5,1)$   & 1.06 & 0.034 & 1.06 & 0.059   \\
 		$(2,15,1)$ & 1.16 & 0.052 & 1.04 & 0.075 \\
 		\midrule 
 		$(3,5,1)$ & 1.02 & 0.019 & 1.04 & 0.048  \\
 		$(3,15,1)$ & 1.04 & 0.027 & 1.04 & 0.040 \\
 		\midrule 
 		$(4,5,1)$ & 1.00 & 0.000 & 0.98 & 0.020 \\
 		$(4,15,1)$ & 1.02 & 0.020 & 1.12 & 0.046 \\ 
 		\midrule \midrule 
 		\multicolumn{5}{c}{Deep Architectures} \\
 		\midrule 
 		& \multicolumn{2}{c}{$\beta_0$} & \multicolumn{2}{c}{$\beta_{n_0-1}$}
 		\\\cmidrule(lr){2-3}\cmidrule(lr){4-5}
 		& Average & SE & Average & SE \\\midrule
 		$(2,5,5,1)$    &1.16 &0.05 & 1.14 & 0.13   \\
 		
 		$(2,15,15,1)$  &1.14 &0.05 &1.00 &0.10 \\\midrule
 		$(3,5,5,1)$ & 1.00 & 0.00 & 1.52 & 0.14 \\
 		$(3,13,13,1)$ & 1.04 & 0.028 & 1.34& 0.15 \\	\midrule 
 		$(4,5,5,1)$& 1.00 & 0.00 & 0.82 & 0.07 \\
 		$(4,8,8,1)$& 1.04 &0.03 & 1.88& 0.30 \\ \bottomrule
 	\end{tabular}
 	\label{tab:bettis}
 \end{table}
 
 \begin{table}[h]
 	\centering 
 	\caption{Average number of bounded and unbounded components of the decision boundary dependent on architecture.}
 	\begin{tabular}{lcccl}\toprule
 		\multicolumn{5}{c}{Shallow Architectures} \\
 		\midrule 
 		& \multicolumn{2}{c}{Unbounded} & \multicolumn{2}{c}{Bounded}
 		\\\cmidrule(lr){2-3}\cmidrule(lr){4-5}
 		& Average & SE & Average & SE \\\midrule
 		$(2,5,1)$  &1.00 & 0.070 & 0.06 & 0.034    \\
 		$(2,15,1)$  & 0.88 & 0.073 &  0.16 & 0.052  \\
 		$(3,5,1)$  & 1.02 & 0.053 & 0.02 & 0.020  \\
 		$(3,15,1)$ &1.00 &  0.040& 0.04 & 0.023 \\
 		$(4,5,1)$ & 0.98 & 0.020 & 0.00  & 0.00 \\
 		$(4,15,1)$ &1.10 & 0.042 & 0.02 &0.020 \\ \midrule \midrule 
 		\multicolumn{5}{c}{Deep Architectures} \\
 		\midrule 
 		& \multicolumn{2}{c}{Unbounded} & \multicolumn{2}{c}{Bounded}
 		\\\cmidrule(lr){2-3}\cmidrule(lr){4-5}
 		& Average & SE & Average & SE \\\midrule
 		$(2,5,5,1)$   & 0.98 & 0.140 & 0.16 & 0.052   \\
 		$(2,15,15,1)$ & 0.86 & 0.100 & 0.14 & 0.049 \\
 		$(3,5,5,1)$ & 1.52 & 0.142 & 0.00 & 0.000  \\
 		$(3,10,10,1)$ & 1.38 & 0.155 & 0.04 & 0.027\\
 		$(4,5,5,1)$ & 0.82 & 0.0730 & 0.00 & 0.000 \\
 		$(4,8,8,1)$ & 1.84 & 0.301 & 0.04 &0.040 \\\bottomrule
 	\end{tabular}
 	\label{tab:ccs}
 \end{table}

These observations lend additional credence to the notion that depth has a stronger influence than width on the topological complexity that a network can be easily trained to express.

 \section{Conclusion and further directions}
 \label{sec:conclusion}
  We have presented a new combinatorial characterization of ReLU network functions, and demonstrate the utility of this characterization for obtaining topological information about networks which was previously difficult to access. The experiments illustrate the utility of this characterization for driving further experimental research on the properties of ReLU networks in different conditions. Furthermore, we believe this framework proposed could be used to derive additional properties of ReLU networks.

 Theoretically, we have provided a foundation to study the structure of the canonical polyhedral complex across all dimensions. In practice, a drawback to using this algorithm in empirical work is that the algorithm is still an exponential process in the input dimension, so realistically only low-dimensional slices of the true decision boundary of a network can be investigated empirically.

 We primarily believe that this tool can be useful to theoreticians, in that it provides a local characterization of vertices of $\mathcal{C}(F)$, for example in obtaining distributional properties of local curvature which rely on a characterization of the low-dimensional connectivity. This local characterization makes piecewise linear analogs of Morse theory as in \cite{grunert} more accessible to apply to $\mathcal{C}(F)$ by providing a local combinatorial identification between neighborhoods of vertices of $\mathcal{C}(F)$ and neighborhoods of coordinate axes.

 One application of this work is to analyze topological generalization of networks. A key indicator of a network's generalizability is whether its sublevel sets have the appropriate topological properties \cite{Bianchini2014}. In addition, the architecture of a classification network influences the topology of the expressible decision boundaries of that network \cite{guss}. Empirically, topological data analysis provides practical approximation for low-dimensional features of high dimensional data. While approximations exist to obtain the topological properties of a network's decision boundaries using topological data analysis \cite{homologydb}, these properties are dependent on the geometry of the network function. In places where the network's decision boundary has high curvature, the approximation methods may lead to inaccuracy between the true topology of a network's decision boundary and the topology which is approximated by persistent homology methods, but it is precisely those locations where a network is vulnerable to adversarial examples \cite{db_curvature}. A measure of the true topology of the decision boundary could provide a metric for comparison.
 
 \begin{figure}[H]
 	\centering 
 	\includegraphics[width=\linewidth]{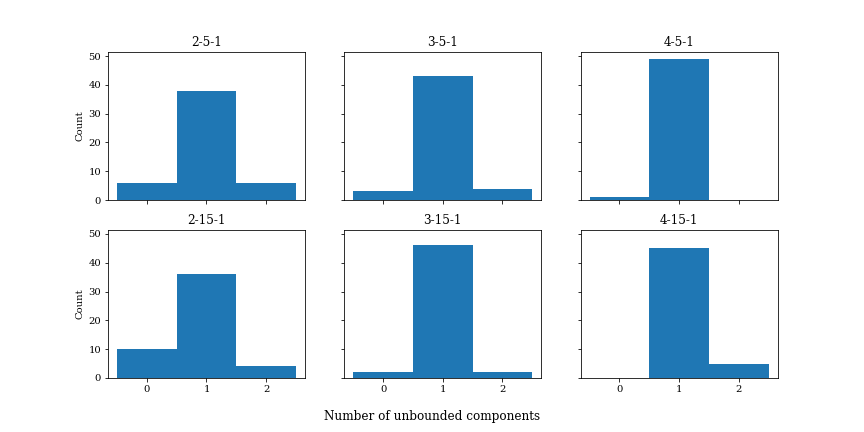}
 	\includegraphics[width=\linewidth]{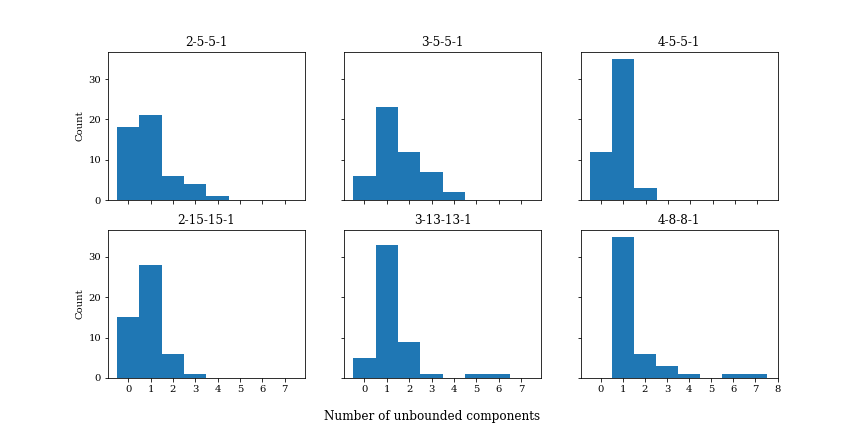}
 	\caption{Distribution of the total number of connected components of the decision boundary for selected architectures.}
 	\label{fig:ccs}
 \end{figure}

We lastly believe it is possible to obtain an explicit understanding of the evolution of a network's decision boundary through a training path.  We conjecture that the changes in the facet structure of $\mathcal{C}(F)$ should be "discrete" in that they would only change at finitely many locations in a general training path, which would thus open an avenue for tracking the vertices of $\mathcal{C}(F)$ as the network trains in parameter space, and establishing theoretical limits on the possible discrete changes that may occur to the set of vertices of $\mathcal{C}(F)$ during training.


\section*{Acknowledgements}
We thank J. Elisendra Grigsby, Kathryn Lindsey, Dev Sinha and Jean-Baptiste Tristan for their insightful comments and questions. 

This work benefited from access to the University of Oregon high performance computer, Talapas.
                              
\bibliographystyle{unsrt}
\bibliography{bibliography}


\newpage 
\appendix

\section{Mathematical Background and Proofs}

\subsection{Mathematical Background} \label{A:Definitions}

A summary of relevant results and definitions about polyhedral complexes and complexes arising from affine hyperplane arrangements found in \cite{gloriginal}, Sections 2-3, and and \cite{grunert}, Sections 1-2, with the most relevant information repeated below. Namely, we make repeated use of polyhedral geometry: 

\begin{defn}[Polyhedra, Polyhedral Complex, cf. \cite{gloriginal}] \label{def:polycomplex}
	$\;$
	
\begin{itemize}
	\item A \textbf{polyhedron} is an intersection of the form $\bigcap_{1\leq i\leq m} H_i^+$ for some set of (codimension 1) hyperplanes $H_1, ... , H_m \subset \mathbb{R}^n$.
	
	\item We will say a point is on the \textbf{interior} of a polyhedron $P$ if it is on the interior of $P$ with respect to the subspace topology of the affine span of $P$, except when $P$ is a point, in which case by convention its interior is nonempty. We use the notation $P^\circ$ to denote the interior of $P$.
	 
	\item A \textbf{face} of a polyhedron $P$ embedded in $\mathbb{R}^n$ is a set of the form $H \cap P$, where $H$ is a codimension 1 hyperplane in $\mathbb{R}^n$ and $H \cap P$ does not contain any interior point of $P$. (The empty set is a face of any polyhedron.)  A hyperplane which intersects $P$ in a nonempty face is called a \textbf{supporting hyperplane} of $P$. All other hyperplanes which intersect $P$ are called \textbf{cutting hyperplanes} of $P$. We denote the relation "$C$ is a face of $D$" with $C \leq D$.
	
	\item A \textbf{polyhedral complex embedded in $\mathbb{R}^n$} is a set of polyhedra contained in $\mathbb{R}^n$ which is (a) closed under taking faces, (b) closed under intersection, in that the intersection $P_1 \cap P_2$ is the (unique) maximal common face of $P_1$ and $P_2$, which may be empty. 
\end{itemize} 
\end{defn} 

Under this definition, polyhedra may not be bounded, but they are always closed. As a result, an affine hyperplane arrangement induces a natural polyhedral decomposition on its ambient space. We use the following notation for consistency with previous work, providing a notation for affine hyperplane arrangement $R^{(i)}$ associated to an affine map $A_i$. 

\begin{defn}[$R^{(i)}$, $\pi_j$, cf. \cite{gloriginal}, Definition 6.7]\label{def:harrangement}
	Let $A_i: \mathbb{R}^{n_{i-1}}\to \mathbb{R}^{n_i}$ be an affine function for $1 \leq i \leq n$. Denote by $R^{(i)}$ the polyhedral complex associated to the hyperplane arrangement in $\mathbb{R}^{n_{i-1}}$, induced by the hyperplanes given by the solution set to $H_{ij} = \{x \in \mathbb{R}^n:\pi_j \circ A_i(x)=0  \}$, where $\pi_j$ is the linear projection onto the $j$th coordinate in $\mathbb{R}^{n_{i}}$. 
\end{defn}

Continuing to the framework in \cite{gloriginal}, we investigate the following class of neural network functions.

\begin{defn}[\cite{gloriginal}, Definition 2.1]\label{def:ReLU}
	Let $n_0,...,n_m \in \N$. A \textbf{fully-connected ReLU neural network} with \textbf{architecture} $(n_0, ..., n_m,1)$ is a collection 
	$\mathcal{N} = \{ A_i\}$ of affine maps $A_i: \R^{n_i} \to \R^{n_{i+1}}$ for $i=0,...,m$. Such a collection
	determines a function $F_{\mathcal{N}} : \R^{n_0}\to \R$, the \textit{\textbf{associated neural network map}}, given by the composite 
	
	$$\R^{n_0} \xrightarrow{F_1 = \relu \circ A_1} \R^{n_1} \xrightarrow{F_2 = \relu \circ A_2} \R^{n_2} \xrightarrow{F_3 = \relu \circ A_3} ... \xrightarrow{F_m = \relu \circ A_m} \R^{n_m} \xrightarrow{G = A_{m+1}} \R^{1} $$
	
	where ReLU refers to the function $\max\{0,x\}$ applied pointwise \cite{relu}. We say that this network has \textbf{depth} $m+1$ and \textbf{width} $\max\{n_1,...,n_m,1\}$. The maps $F_k$ are called the $k$th \textbf{layer maps}.
\end{defn}

As a piecewise-affine linear function, $F_\mathcal{N}$, which we simplify to $F$, defines an obvious polyhedral decomposition of input space, namely into the (largest) polyhedra on which it is affine-linear. However, in \cite{gloriginal} the authors show the utility of considering not only the decomposition which $F$ itself defines, but the common refinement of decompositions by intermediate composites.

\begin{defn}\label{def:lastlayer}
	If  $F = G \circ F_{m}\circ ... \circ F_1$ is a ReLU neural network with $F:\mathbb{R}^{n_0}\to \mathbb{R}$ , then we denote: 
	
	$$F_{(k)} = F_k \circ ... \circ F_1 $$ 
	
	and write that this is \textbf{$F$ ending at the $k$th layer}.
	
	Likewise, we denote
	
	$$F^{(k)} = G \circ F_{m} \circ ... \circ F_k $$
	
	and call $F^{(k)}: \mathbb{R}^{n_{k-1}}\to \mathbb{R}$ by \textbf{$F$ starting at the $k$th layer}. 
\end{defn}

Thus $F = F^{(k)} \circ F_{(k-1)}$ for any $k$. A definition for the canonical polyhedral complex $\mathcal{C}(F)$ through a universal property can be given as \textit{the common refinement of the polyhedral decomposition of input space such that all $F_{(i)}$ are affine linear on cells}. For implementation, we prefer a definition through explicit identification of cells, using further language from \cite{gloriginal}:

\begin{defn}[\cite{gloriginal}, definition 8.1 ]\label{def:nodemap}
	If $F$ is a ReLU neural network, the \textbf{node map} $F_{i,j}$ is defined by: 
	
	$$\pi_j \circ A_i \circ F_{i-1} \circ ... \circ F_1: \R^{n_0} \to \R$$ 
\end{defn}

In particular, the locus in input space where $F_{ij}=0$ is of particular interest, and to draw analogies to hyperplane arrangements, we use the phrase "bent hyperplane." 

\begin{defn}[\cite{gloriginal}, Definition 6.1 ]\label{def:benthyperplane}
	A \textbf{bent hyperplane} of $\mathcal{C}(F)$ is the preimage of $0$ under a node map, that is, $F_{ij}^{-1}(0)$ for fixed $i,j$. 
\end{defn}

A bent hyperplane can contain polyhedral regions with codimension less than one, but this occurs with zero probability. The conditions under which the bent hyperplanes' maximal cells are always codimension 1 are listed in \cite{gloriginal}.  

 The formal definition of the canonical polyhedral complex $\mathcal{C}(F)$ is defined in \cite{gloriginal} using the notion of a "level set complex," defined in \cite{grunert}. We streamline the definition, working more directly in two ways, as needed below. 

\begin{defn}[Canonical Polyhedral Complex $\mathcal{C}(F)$, cf. \cite{gloriginal}, Definition 6.7]\label{def:canonicalcomplex}
	Let $F: \mathbb{R}^{n_0} \to \mathbb{R}$ be a ReLU neural network with $m$ layers. Define $\mathcal{C}(F)$ as follows: 
	\begin{enumerate}
		\item ("Forward Construction") Define $\mathcal{C}(F_{(1)})$ by $R^{(1)}$ (Definition \ref{def:harrangement}). Then let $\mathcal{C}(F_{(k)})$ be defined in terms of $\mathcal{C}(F_{(k-1)})$ as the polyhedral complex consisting of the following cells: 
		
		$$\mathcal{C}(F_{(k)}) = \left\{ C \cap F_{(k-1)}^{-1}(R) : C \in \mathcal{C}(F_{(k-1)}), R \in R^{(k)} \right \} $$
		
		Then $\mathcal{C}(F)$ is given by $\mathcal{C}(F_{(m)})$. 
		
		\item ("Backwards Construction") Define $\mathcal{C}(F^{(m)})$ by $R^{(m)}$. Then $\mathcal{C}(F^{(k-1)})$ can be defined from $\mathcal{C}(F^{(k)})$ as the polyhedral complex consisting of the following cells: 
		
		$$\mathcal{C}(F^{(k-1)}) = \left\{ R \cap F_{k-1}^{-1}(C) : R \in R^{(k-1)}, C \in \mathcal{C}(F^{(k)}) \right\}$$
		
		Then $\mathcal{C}(F)$ is given by $\mathcal{C}(F^{(1)})$. 
	\end{enumerate}
\end{defn}

That the two are equivalent follows from the distributivity of function preimage over set intersection (or, more generally, associativity of pullbacks). That the resulting construction is indeed a polyhedral complex is discussed thoroughly in \cite{gloriginal} and \cite{grunert}, chapter 2. In particular, we make use of the following lemma regarding boundary relations, rewritten so as to not require additional notation.

\begin{lemma}[cf. \cite{grunert}, Lemma 2.4]\label{lem:facerelations}
	Let $M\subseteq \mathbb{R}^m$ and $N \subseteq \mathbb{R}^n$ be polyhedral complexes and $f: |M| \to \mathbb{R}^n$ be continuous and affine on cells of $M$. Let $\leq$ denote face relations in the respective polyhedral complexes. If $C \leq C'$ are polyhedra in $M$, and $D \leq D'$ are polyhedra in $N$, then 
	
	$$C\cap f^{-1}(D) \leq C' \cap f^{-1}(D') $$ 
	
	is a face relation in the polyhedral complex consisting of the cells $\{ C \cap F^{-1}(D) : D \in N, C \in M \}$. 
\end{lemma}

The notion of "transversality on cells," defined in \cite{gloriginal}, will be critical for the next section. 

\begin{defn}[\cite{gloriginal}, Definition 4.5 ]\label{def:transverseoncells}
	Let $X$ be a polyhedral complex of dimension $d$ in $\R^n$, let $f:|X|\to \R^{r}$ be a map which is smooth on all cells of $X$ and let $Z$ be a smoothly embedded submanifold (without boundary) of $\R^r$. We say $f$ is \textbf{transverse on cells of $X$} to $Z$ and write $f\pitchfork_X Z $ if the restriction of $f$ to the interior $C^\circ$ of every $k$-cell $C$ of $X$ is transverse to $Z$ when $0 \leq k \leq d$ 
\end{defn}

Lastly, we will need the following notions of \textit{generic} regarding hyperplane arrangements and neural networks, respectively. 

\begin{defn}[\cite{gloriginal} Definitions 2.7, 2.9] \label{def:generic}
	A hyperplane arrangement in $\mathbb{R}^n$ is called \textbf{generic} if each all sets of $k$ hyperplanes intersect in an affine space of dimension $n-k$.  
	A neural network is called \textbf{generic} if all of its affine maps $A_i$ have generic corresponding hyperplane arrangements, $R^{(i)}$. 
\end{defn}

In \cite{gloriginal} it is established that the union of bent hyperplanes of $\mathcal{C}(F)$ form the $(n_0-1)$-faces of $\mathcal{C}(F)$ with probability 1. In the next section we expand on this characterization for lower-dimensional subcomplexes.

The reader is referred to \cite{ghrist_elementary} for a brief background in algebraic topology, especially the notions of homology and cohomology, chain complexes, and duality (sections 4-6).

\subsection{New Definitions and Proofs}

We begin by defining a key additional condition on neural networks, which is necessary for many of the results in this section to hold. 

\begin{defn}\label{def:supertransversal}
	Let $F$ be a ReLU neural network of depth $m$. Let $\textbf{F}^{(i)}:\R^{n_i}\to \R$ be the neural network defined by the last $m-i$ layers of $F$ as in definition \ref{def:lastlayer}. Suppose, for all $1\leq i\leq n$, $F_i$ is transverse on cells of $R^{(i-1)}$ to the interior of all cells of $\mathcal C(F^{(i)})$. Then we call $F$ a \textbf{supertransversal} neural network.
\end{defn}

The condition of network supertransversality is stronger than the notion of network transversality in \cite{gloriginal} (Definition 8.2), but it still holds on a full-measure subset of parameter space, as we  show here.


\begin{lemma}
	\label{lem:generic}
	Supertransversality is full measure in $\mathbb{R}^P$, where $P$ is the set of network parameters. 
\end{lemma} 
\begin{proof}
	First, a single-layer neural network $F^{(m)}: \mathbb{R}^{n_m} \to \mathbb{R}$ is trivially supertransversal; $\mathbb{R}$ has one cell which is already full dimension. 
	
	Next suppose that $F^{(k)}$ is supertransversal, and let $F_{k-1}: \mathbb{R}^{n_{k-1}}\to \mathbb{R}^{n_k}$ be a network layer. 
	
	Suppose it is the case that $F_{k-1}$ is nontransverse on some cell $R$ of $R^{(i-1)}$ to some cell $C$ of $\mathcal{C}(\textbf{F}^{(k)})$. If so, it must be the case that $F_{k-1}(R)\cap C$ is nonempty and $T(F_{k-1}(R)) \oplus T(C) \neq \mathbb{R}^{n_{k}}$. Call $T(F_{k-1}(R)) \oplus T(C)$ by $T_{R,C}$. If $T_{R,C} \neq \mathbb{R}^{n_{k-1}}$, then it is instead a vector subspace of less than full rank, and an affine translation of $\mathcal{C}(F^{(k)})$ by any vector in $\mathbb{R}^{n_{k}} - T_{R,C} $ will ensure  $F_{k-1}(R)\cap C$ is subsequently empty, since $F$ is affine on $R$ and $C$ contained in an affine subspace of $\mathbb{R}^{n_k}$. 
	
	Let $\delta_{R,C}$ be the minimum distance between pairs of points in $F_{k-1}(R)$ and $C$. Since these cells are closed (though not necessarily compact), this is well defined, and furthermore if $R \cap C =\emptyset$, then $  \delta_{R,C}>0$. Let $\delta= \{\min \delta_{R,C}: F_{k-1}(R)\cap C =\emptyset \}$. Then $\delta>0$. 

	Since there are finitely many cells $R$ and $C$, the set $$\mathbb{R}^{n_{k-1}} - \bigcup_{R,C} T_{R,C} $$ is generic in $\mathbb{R}^{n_{k-1}}$, where the union is taken over only those cells where $T_{R,C}$ is not full rank. An affine translation of $\mathcal{C}(F^{(k)})$ by any vector in this set with magnitude greater than 0 but less than $\delta$ yields a supertransversal network. 
	
	Since this can be performed regardless of the weights and biases of $F_k$, and an affine translation of the input space of $F^{(k)}$ does not change its supertransversality properties, the network $F_{k-1} \circ F^{(k)}$ is supertransversal on a full-measure subset of parameter space, which completes our inductive step.
\end{proof}

The combinatorial characterization of $\mathcal{C}(F)$ is through the following combinatorial construction called \textit{sign sequences}. The primary use of these sign sequences is to track face relations. 

\begin{defn}\label{def:signsequence}
	Define $s: \mathcal{C}(F) \to \{-1,0,1\}^N$ by $s_{ij}(C) = \sgn(F_{ij}(C))$. We call $s(C)$ the \textbf{sign sequence} of the cell $C$.
\end{defn}

This construction is used in the theory of oriented matroids and hyperplane arrangements, cf. \cite{hyperplanes}, and in particular the construction may be used to identify polyhedra in an affine hyperplane arrangement by denoting which halfspaces and hyperplanes were intersected to form that region. However, any analogous properties for ReLU networks must be proven, as many of the properties below fail to hold for arbitrary PL manifold arrangements. We must show that the construction still provides a combinatorial description of the polyhedra of the network:

\begin{theorem}\label{thm:injective}
	The function $s$ is well-defined and injective on cells of $\mathcal{C}(F)$.
\end{theorem}
\begin{proof}
	To see that $s$ is well-defined, suppose $x_1, x_2 \in \R^{n_0}$ are such that $\sgn (F_{ij}(x_1)) \neq \sgn( F_{ij}(x_2))$ for some $i,j$. We wish to show that $x_1$ and $x_2$ are not in the same cell of $\CF$. However, we see that the images $F_{{i-1}}\circ ... \circ F_1(x_1)$ and $F_{i-1}\circ ... \circ F_1(x_2)$ cannot be in the same cell of of $R^{(i)}$ (the induced polyhedral decomposition of $\R^{n_{i-1}}$ by $\textbf{A}_i$), because they differ in their location relative to the $j$th hyperplane. Thus $x_1$ and $x_2$ are in different cells of $\mathcal{C}(F_{(i)})$. As $\mathcal{C}(F)$ is a further polyhedral subdivision of $\mathcal{C}(F_{(i)})$, $x_1$ and $x_2$ are in different cells of $\mathcal{C}(F)$. So, $s$ is well defined.

	Next, suppose $C_0$ and $C_1$ are cells such that $s(C_0)=s(C_1)$. Let $x_0 \in C_0$ and $x_1 \in C_1$. We wish to show $C_0 = C_1$. We proceed by induction on layers in the forward direction.

	We show first that, as a base case for induction, $x_0$ and $x_1$ are in the same cell in $\mathcal{C}(F_1)$. 
	
	Indeed since $s(C_0)=s(C_1)$, $x_0$ and $x_1$ are contained in the same cell of $R^{(1)}$, following the corresponding fact for hyperplane arrangements, 
	this immediately means that $x_0$ and $x_1$ are in the same cell of $\mathcal{C}(F_1)$.
	
	Now suppose as an inductive hypothesis that $x_0$ and $x_1$ are in the same cell of $\mathcal{C}(F_{(k)})$.
	Call $y_0 = F_{(k)}(x_0)$ and $y_1 = F_{(k)}(x_1)$. 
	Because $\sgn(F_{(k+1)j}(x_1))=\sgn(F_{(k+1)j}(x_2))$ for all $0 \leq j \leq n_{k}$, this implies that $y_0$ and $y_1$ are in the same intersection of halfspaces and hyperplanes in the co-oriented hyperplane arrangement $\textbf{A}_{k+1}$, that is, the same cell of $R^{(k+1)}$. Therefore, as $x_1$ and $x_2$ are in the same cell $C$ of  $\mathcal{C}(F_{(k)})$ and their image is in the same cell $C'$ of $R^{(k+1)}$ we conclude $x_0$ and $x_1$ are in the same cell in $\mathcal{C}(F_{(k+1)})$ given by 
	
	$$C \cap (F_{(k+1)})^{-1}(C')$$
	
	That this is a unique polyhedral cell in $\mathcal{C}(F_{(k+1)})$, follows the work in \cite{grunert}, Lemma 2.5. By induction, as $F$ is composed of finitely many layers, $x_0$ and $x_1$ are in the same cell of $\CF$. 
\end{proof}

\begin{figure}[h]
	\centering 

	\caption{Even if $M_i$ are codimension-1 connected co-orientable PL manifolds embedded in $\mathbb{R}^n$ whose intersection subdivides $\mathbb{R}^n$ into polyhedral regions, and the resulting polyhedral subdivision is dual to a cubical complex, it is possible that labeling each region by its location relative to the co-orientation of those manifolds fails to be injective. This example depicts 3 PL submanifolds in $\mathbb{R}^2$ whose embedding has the aforementioned properties, but there are too many cells (14 vertices, 30 edges, and 17 2-gons) to be labeled by the 27 possible labelings in $\{-1,0,1\}^3$. } 
	\includegraphics[width=2in]{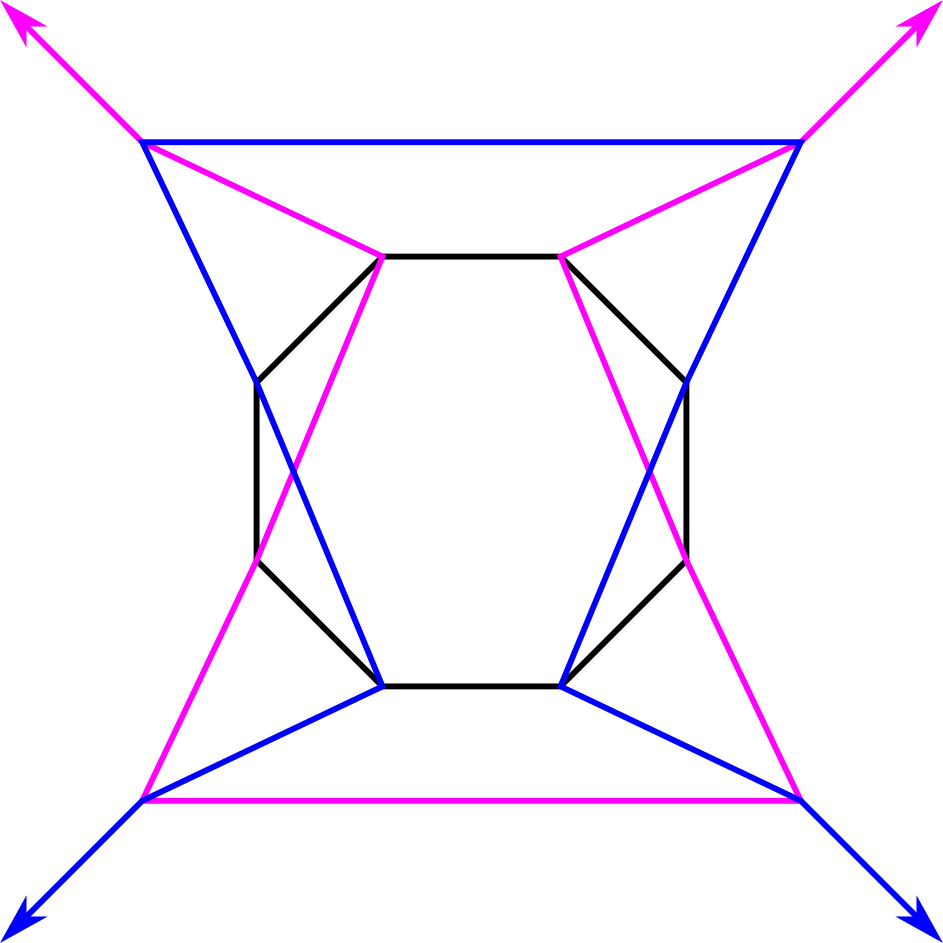}
	\label{fig:noninjectivity}
\end{figure} 

The injectivity of $s$ is special to constructions arising from hyperplane arrangements, and not general manifold arrangements. Indeed, even if a set of codimension-1 PL submanifolds subdivide $\mathbb{R}^n$ in a way which this injectivity fails; see Figure \ref{fig:noninjectivity}

Below we show that C(F) can be reconstructed from vertices and sign sequences. It is more traditional to work "top down," considering the top dimensional polytopes and their facets. But unlike theories such as hyperplane arrangements and oriented matroids, there is no guarantee that knowing the sign sequences of the top-dimensional regions allows one to deduce the sign sequences of the zero-dimensional regions (circuit-cocircuit duality does not hold). The following example is an explicit illustration of this fact.  

\begin{theorem}
	\label{thm:topdimbad}
	There exists a pair of networks $F_1$ and $F_2$ such that the set of strings encoding the activation patterns in the interiors of the cells of $\mathcal{C}(F_1)$ and $\mathcal{C}(F_2)$ are equal, but the polyhedral complexes $\mathcal{C}(F_1)$ and $\mathcal{C}(F_2)$ are not combinatorially equivalent.
\end{theorem} 

\begin{figure}[H]
	\centering
	\caption{The two canonical polyhedral complexes pictured below have differing combinatorics and differing topology of their decision boundaries, but the set of sign sequences of the top dimensional regions is equal (see Theorem \ref{thm:topdimbad}). Explicit weights and biases for this construction are available in the code provided at the online repository.}
	\label{fig:topregions}
	
	\includegraphics[width=0.4\linewidth]{"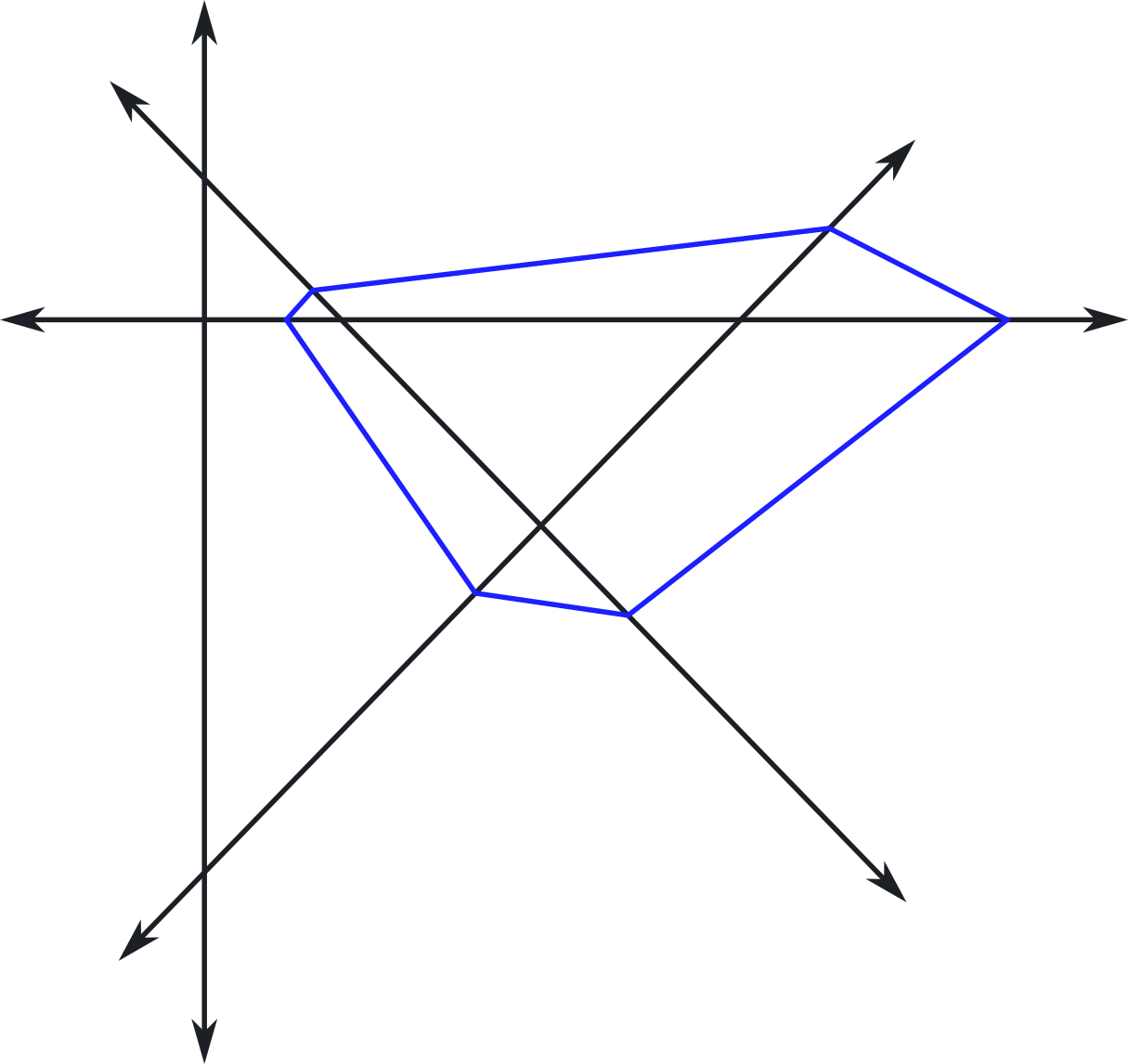"} \includegraphics[width=0.4\linewidth]{"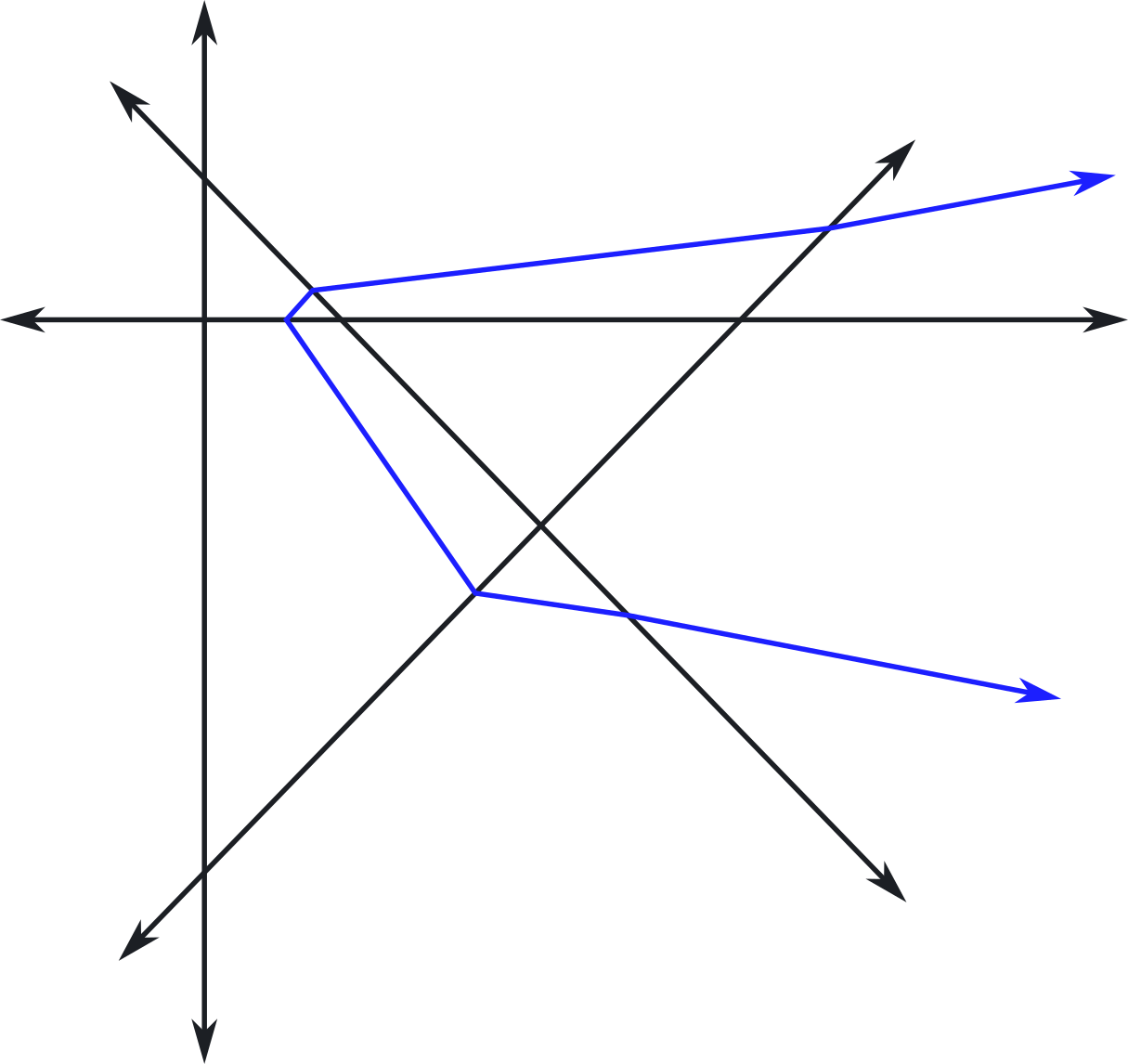"}
\end{figure} 

\begin{proof}
	We provide an example in Figure \ref{fig:topregions}.  Two neural networks $F_1: \R^2 \to \R^4 \to \R$ and $F_2: \R^2 \to \R^4 \to \R$ have the pictured canonical polyhedral complexes. One has a bounded decision boundary, whereas the other is unbounded. However, both have identical  sign sequences of their top-dimensional regions, given by: 
	
	\begin{tabular}[H]{ccccc}
		(-1, 1, -1, 1, 1), & (-1, 1, 1, -1, 1), &(1, -1, -1, -1, 1), &(-1, 1, -1, 1, -1), &(-1, -1, 1, -1, 1), \\
		(-1, 1, 1, 1, -1),& (1, -1, -1, 1, -1), &(1, -1, 1, -1, 1), &(-1, -1, 1, 1, -1), &(-1, 1, 1, 1, 1), \\
		(-1, -1, -1, 1, -1),& (1, -1, -1, 1, 1),& (-1, -1, 1, 1, 1),& (1, -1, 1, 1, -1), &(1, -1, 1, 1, 1), \\
		(1, 1, -1, 1, -1), &(1, 1, -1, 1, 1) & & & 
	\end{tabular} 

That these canonical polyhedral complexes have the same sign sequences of their top-dimensional regions is more easily seen by looking at the differences between the two pictures, which only depend on the blue "decision boundary." Each region which is subdivided into two regions by the blue bent hyperplane in the left image is also subdivided by the blue bent hyperplane in the right image.
	
\end{proof}

As illustrated in \cite{gloriginal}, it is not always the case that the preimages $F_{ij}^{-1}(0)$ are $n_0-1$ dimensional. In order to establish dimension, supertransversality and genericity are key.  

\begin{lemma}\label{lem:countzeros}
	Let $F$ be generic and supertransversal. Let $C$ be a $k$-cell of $\CF$. Then $s(C)$ has exactly $n_0-k$ entries which are zero. That is, $C$ is contained in the intersection of $n_0-k$ bent hyperplanes.
\end{lemma}

\begin{proof}
		This is certainly true for any neural network of the form $G \circ F_m$ satisfying the condition that $F_m$ is generic as a layer map, as the bent hyperplane arrangement is equal to the hyperplane arrangement, which is a generic hyperplane arrangement (Definition \ref{def:generic}). 
		
		We proceed via induction, using the backwards construction of $\mathcal{C}(F)$ (Definition \ref{def:canonicalcomplex}). 
		
		Suppose by way of induction $F^{(i)}=G \circ F_m \circ ... \circ F_i$, and that $\mathcal{C}(F^{(i)})$ satisfies the condition that $C \in \mathcal{C}(F^{(i)})$ is a $k$-cell if and only if $C$ is contained in exactly $n_i - k$ bent hyperplanes. 
		
		Now, suppose $F_{i-1}$ is transverse on the cells of $R^{(k-1)}$ to $C$ for all $C$ in $\mathcal{C}(\textbf{F}^{(i)})$. Consider $\textbf{F}^{(i-1)}= G\circ F_m \circ ... \circ F_i\circ F_{i-1}$. 
		
		Let $C'$ be a cell in $\mathcal{C}(\textbf{F}^{(i-1)})$. Note that by definition, $C'$ is given by $F_1^{-1}(C) \cap D$ for some $C \in \mathcal{C}(\textbf{F}^{(i)})$ and some minimal (by inclusion) cell $D$ of $R^{(i-1)}$. In particular, we may assume $C'$ is not contained in any proper face of $D$. If $D$ has codimension $\ell$, then $D$ is in the intersection of exactly $\ell$ hyperplanes in $R^{(i-1)}$ by the genericity of $F$. Because $F_1$ is transverse on cells of $R^{(i-1)}$ to $C$, codim$(C')$ in the interior of $D$ is equal to codim$(C)=k$, with total codimension $k+\ell$. 
		
		As $C$ is contained in the intersection of exactly $k$ bent hyperplanes in $\mathcal{C}(F^{(i)})$, $C'$ is contained in the intersection of the preimage of precisely those same $k$ bent hyperplanes in $\mathcal{B}(\textbf{F}^{(i-1)})$. Furthermore $C'$ is contained in the intersection of the $\ell$ hyperplanes in $R^{(i-1)}$ which intersect to form $D$, and no additional hyperplanes as $C'$ is not contained in any proper face of $D$. Thus $C'$ is contained in the intersection of precisely $k + \ell$ bent hyperplanes in $\mathcal{C}(\textbf{F}^{(i-1)})$ and has codimension $k+\ell$. The number of zeros in $s(C')$ must be equal to the number of  bent hyperplanes it is contained in, by definition. This completes the inductive step.
\end{proof}


In order to establish additional properties regarding  the existence of cells satisfying certain relations in supertransversal networks, we rely on the following lemma. 

\begin{lemma}\label{lem:facerelations_transversality}
	Let $f:M\to \mathbb{R}^n$ be a PL map affine on cells of an embedded polyhedral complex $M\subset \mathbb{R}^m$. Let $N$ be a polyhedral complex embedded in $\mathbb{R}^n$. Suppose $f$ is transverse on cells of $M$ to the interior of all cells of $N$. 
	
	If $C \leq C'$ is a face relation in $M$, $D \leq D'$ is a face relation in $N$, and $f(C^\circ) \cap D$ is nonempty, then $f(C'^\circ) \cap D'^\circ$ is nonempty.  
\end{lemma}

\begin{proof}
	First we show that $f(C^\circ) \cap D'^\circ$ is nonempty. If $D=D'$ then we are done. Otherwise, consider $f(C)$, which is the image of the polyhedron $C$ under an affine map. If the affine span $A$ of $f(C)$ does not intersect the interior of $D'$, then $A \cap D'$ is contained in a proper face $E'$ of $D'$. In this case, $T(f(C))\oplus T(E') \neq \mathbb{R}^n$ , and $f$ is not transverse on $C$ to $E'$. Therefore, $A\cap D'^\circ$ is nonempty. 
	
	Within $A$ we have $\partial(A \cap D') \subseteq A \cap \partial(D')$. As a result letting $x \in C^\circ$ and $f(x) \in D' $, every open neighborhood of $f(x)$ in $A$ must have nontrivial intersection with the interior of $D'$. Take an open neighborhood $N$ of $x$ in $C^\circ$. We note that $f: C \to A$ is a submersion (locally a surjective linear map). Thus $f(N)$ is open in $A$, and $N \cap D'^\circ$ must be nonempty, so $C^\circ \cap f^{-1}(D'^\circ)$ is nonempty. 
	
	To see that $D'^\circ\cap f(C'^\circ)$ is nonempty, take $x \in C^\circ$ with $f(x)\in D'^\circ$. If $N$ is a neighborhood of $f(x)$ in $D'^\circ$ then $f^{-1}(N)$ must be an open neighborhood of $x$ in $M$, containing $x \in C$. As $C \subset \partial C'$, $f^{-1}(N)\cap C'$ is nonempty, so $f(C'^\circ)\cap D'$ is nonempty. 
\end{proof}


Now we can define a key algebraic structure which will lead us to be able to deduce the structure of $\mathcal{C}(F)$ in general, in particular an algebraic structure which allows us to generate all sign sequences which are present from the sign sequences of the vertices. The following holds for all supertransversal networks (and does not rely on the layer maps being generic). 

\begin{lemma}\label{lem:faceproduct}
	Let $F$ be a supertransversal neural network. 
	
	If $C$ and $D$ are two cells of $\mathcal{C}(F)$, the product $S(C)\cdot S(D)$ given by: 
	
	$$(S(C)\cdot S(D))_{ij} =\begin{cases}
	S(C)_{ij} & \text{if }S(C)_{ij} \neq 0 \\ 
	S(D)_{ij} & \text{otherwise}
	\end{cases} $$ 
	
	is well-defined as a product between sign sequences. That is, there exists a cell $E$ in $\mathcal{C}(F)$ such that $S(C)\cdot S(D) = S(E)$ for all such cells $C$ and $D$. 
	
	Furthermore, $C \leq E$, that is, $C$ is a face of $E$ or equal to $E$. 
	
	Thus, sign sequences of a supertransversal network form a semigroup.	
\end{lemma}
\begin{proof}
	First, this is true for any single-layer network $F: \mathbb{R}^{n_m} \to \mathbb{R}$, since it is true for hyperplane arrangements; see \cite{hyperplanes}, Section 1.4 for a treatment. 
	
	Now, suppose these properties hold for any supertransversal $k$-layer neural network and inductively, using the backwards construction of $\mathcal{C}(F)$, let  $F$ be a $k+1$-layer supertransversal neural network.  Then $F^{(2)} =  G\circ F_{k+1}\circ ... \circ F_{2} $ is a $k$-layer supertransversal network and our inductive hypothesis holds for $\mathcal{C}(F^{(2)})$. We will denote the sign sequences of cells with respect to $\mathcal{C}(F^{(2)})$ by $S^{(2)}(C)$. 
	
	Let $C$ and $D$ be cells of $\mathcal{C}(F) = \mathcal{C}(F^{(2)}\circ F_1)$. Then $C=R_1 \cap F_1^{-1}(C')$ and $D = R_2 \cap F_1^{-1}(D')$, for cells $R_1, R_2$ in $R^{(1)}$ and cells $C', D' \in \mathcal{C}(F^{(2)})$, by the definition of $\mathcal{C}(F)$. Now, by inductive hypothesis $S^{(2)}(C')\cdot S^{(2)}(D') = S^{(2)}(E')$ for some cell $E'$ in $\mathcal{C}(F^{(2)})$, and $C'$ is a face of $E'$.
	
	 Denote the sign sequences with respect to $R^{(1)}$ as $S_1$. Since $R^{(1)}$ is a polyhedral complex induced by an affine hyperplane arrangement, $S_1(R_1)\cdot S_1(R_2) = S_1(R_3)$ for $R_3$ a region in $R^{(1)}$, and $R_1$ is a face of $R_3$ or equal to $R_3$.
	 
	 Let $E = R_3 \cap F_1^{-1}(E')$. We wish to show that $S(C)\cdot S(D) = S(E)$, and that $C$ is a face of $E$ or equal to it. Now, $S(C)$ is obtained by $S_1(R_1)$ concatenated with $S^{(2)}(C')$, and likewise for the other cells. Since $S_1(R_1)\cdot S_1(R_2)=S_1(R_3) $ by the corresponding hyperplane arrangement and $S^{(2)}(C')\cdot S^{(2)}(D')=S(E')$ by inductive hypothesis, by concatenation this gives $S(C)\cdot S(D)=S(E)$.
	 
	 To see that $E$ is nonempty we must note that since $F$ is supertransversal, $F_1$ is transverse on $R_1$ to $C'$, and apply Lemma \ref{lem:facerelations_transversality} to obtain that $R_3^\circ \cap F_1^{-1}(E')^\circ$ is nonempty.
	
	Lastly, we recall from Lemma \ref{lem:facerelations} that $C' \leq E'$ and $R_1 \leq R_3$ implies that $(F_1^{-1}(C') \cap R_1 ) \leq (F_1^{-1}(E') \cap R_3)$, that is, $C \leq E$. 
\end{proof}

The following properties of the product defined above are analogous to the same properties in hyperplane arrangements. 

\begin{lemma}\label{lem:faceproperties}
	For all supertransversal networks, the following relations hold for all $C$ and $D$ in $\mathcal{C}(F)$, where the relation $\leq$ denotes ``is a face of":
	\begin{enumerate}	
		\item $C \leq D$ if and only if $S(C)\cdot S(D) = S(D)$ 
		\item $S(C) \cdot S(D) = S(D) \cdot S(C)$ if and only if there is a cell $E$ with $D\leq E$ and $C\leq E$. 	
		\item $S(C) \cdot S(D) = S(C)$ if and only if all bent hyperplanes which contain $D$ also contain $C$. 
	\end{enumerate}
\end{lemma}
\begin{proof}
	$\;$
	
	\begin{enumerate}
		\item We have already shown if $S(C) \cdot S(D) = S(D)$ then $ C\leq D$. If $S(C)\cdot S(D) \neq S(D)$ then there is some index where $s_{ij}(C)=\pm 1$ and $s_{ij}(D)=-s_{ij}(C)$. If so, then $C$ and $D$ are sent to opposite sides of some hyperplane in some layer; this cannot occur if $C \leq D$. 
		
		\item Immediate from the previous statement and Lemma \ref{lem:faceproduct}. 
		
		\item $S(C) \cdot S(D) = S(C)$ if and only if for all node maps for which $F_{ij}(C)=0$, we also have $F_{ij}(D)=0$. But this is true if and only if all bent hyperplanes which contain $C$ also contain $D$. 
	\end{enumerate}
\end{proof}


We now assemble these ideas to present a duality between the canonical polyhedral complex of a generic, supertransversal neural network and a pure cubical complex, providing a surprising amount of new structure to the combinatorics of the canonical polyhedral complex. 

\begin{theorem}\label{thm:cubicalcomplex}
	For each generic, supertransversal neural network $F: \mathbb{R}^{n_0}\to \mathbb{R}$ with at least $n_0$ hidden units in the first layer, the image of the map $S: \CF \to \{-1,0,1\}^n$ uniquely defines a pure $n_0$-dimensional subcomplex of the hypercube $[-1,1]^N$ endowed with the product CW structure. We will call this subcomplex $\mathcal{S}(F)$. In the image, the vertices in $\mathcal{S}(F)$ correspond to $n_0$-cells in $\mathcal{C}(F)$, and in general the $k$-cells of $\mathcal{S}(F)$ correspond to codimension-$k$ cells of $\mathcal{C}(F)$. 
\end{theorem} 
\begin{proof}
	Recall that cubical faces of $[-1,1]^N$ (with its product CW structure) can be identified by sequences of $\{-1,0,1\}^N$.
		
	First if $F$ has at least $n_0$ hidden units in the first layer and it is generic, then $\mathcal{C}(F)$ contains vertices as some of its cells, since the intersection of $n_0$ hyperplanes in general position in $\mathbb{R}^{n_0}$ is a point. Since $\mathcal{C}(F)$ is a connected polyhedral complex, if any of its polyhedra have vertices, then all of them do (see \cite{glmsecond}, Corollary 5.29). 
	
	For any $C \in \mathcal{C}(F)$ there is a vertex $v\leq C$. There are $n_0$ coordinates where $S(v)=0$ by Lemma \ref{lem:countzeros}.  Furthermore $S(v)\cdot S(C)= S(C)$ by Lemma \ref{lem:faceproduct}. Thus $S(C)$ is equal to $S(v)$ except those places where $S(v)=0$. But this is equivalent to the condition that the $n_0$-cell $S(v) \in\mathcal{S}(F)$ has $S(C)$ on its boundary. So, every in the image of $S(F)$ is contained in an $n_0$-cube which is also in the image of $S(F)$. (There are no $n_0+1$-cubes in the image of $S(F)$ by Lemma \ref{lem:countzeros}.) Thus the image of $S(F)$ is "pure $n_0$-dimensional" in the sense that every cube in $S(F)$ is a face of an $n_0$-cube in $S(F)$. 

	Next we show that for a given $n_0$-cube in the image of $S(F)$, all its faces are in the image of $S(F)$. Our strategy is to show that there exists an edge corresponding to each possible sign sequence incident to the corresponding vertex. Then we may apply the sign sequence multiplication in Lemma \ref{lem:faceproduct} to obtain all remaining faces. This is equivalent to establishing that a vertex $v$ of $\mathcal{C}(F)$ has $2n_0$ neighboring edges, each of which have a $1$ or $-1$ replacing a single $0$ from $S(v)$. Of course, any vertex must be incident to at least $n_0$ edges since it belongs to a polyhedral complex with domain $\mathbb{R}^{n_0}$, so we show that if there exists an edge incident to $v$ with $S_{ij}(E)=1$ while $S_{ij}(v)=0$, then there also exists an edge with $S_{ij}(E)=-1$ (and, by symmetry, vice versa).
	
	Suppose that this is not the case for some $v$. Then without loss of generality there exists an earliest $(i,j)$ node map satisfying that $F_{ij}(v)=0$ but for all edges $E$ neighboring $v$ in $\mathcal{C}(F)$, $F_{ij}(E)\geq 0$, since for each edge $E$, $S(E)$ differs from $S(v)$ only in one location. Since the edge set of $v$ is nonempty, this implies that $F_{ij}$ cannot be affine on any affine subspace of $\mathbb{R}^{n_0}$ containing $v$ unless $F_{ij}=0$ on that subspace. 
	
	 As $v$ cannot be a vertex of $\mathcal{C}(F_{(i-1)})$, since it is contained in the intersection of fewer than $n_0$ bent hyperplanes before $F_{ij}$, it is contained in the interior of a larger cell $C$ in $\mathcal{C}(F_{(i-1)})$. As a result, $F_{ij}$ is affine on the interior of $C$. But by the previous paragraph, this means that $F_{ij}(C)=0$, and thus $F_{(i)}$ is not transverse on $C$ to a cell contained in $R^{(i)}$, and cannot be transverse on $C$ to any polyhedral subdivision (including $\mathcal{C}(F^{(i)})$). This implies that there is a layer of $F_{(i-1)}$ which fails to be transverse on cells, which is a contradiction.

	 So, if $v$ is a vertex of $\mathcal{C}(F)$, then for each node map $F_{ij}$ such that $F_{ij}(v)=0$, $v$ has an incident edge with $F_{ij}(E)=1$ and an incident edge with $F_{ij}(E)=-1$ by the same argument. We note by supertransversality that $S(E)$ must have $n_0-1$ entries which are zero. Also, since $v$ is incident to $E$, by Lemma \ref{lem:faceproperties}, $S(E)$ must have the same entries as $S(V)$ except possibly where $S(v)=0$. This means $S(E)=S(v)$ except for at the $(i,j)$ coordinate, as required. Since this occurs at all node maps for which $F_{ij}(v)=0$, we are done.  \end{proof}

Once the existing cells in $\mathcal{S}(F)$ have been located, we only need to establish an explicit duality. The majority of the work has already been done. 

\begin{lemma}\label{lem:coboundary}
	 The face poset of $\mathcal{C}(F)$ is the opposite poset of the face poset of $\mathcal{S}(F)$, and the (mod-two) cellular boundary map of $\mathcal{S}(F)$ is dual to the (mod-two) cellular boundary map of $\CF$.
\end{lemma}

\begin{proof}
	 In $[-1,1]^N$, the cells consist of cubes which are uniquely defined by their center, at points given by sequences in $\{-1,0,1\}^N$. The dimension of each cube is given by the number of $0$ entries in this sequence. The cellular boundary of this cube consists of cells one dimension lower, with a $1$ or $-1$ replacing a $0$ in the sign sequence, providing the (mod two) boundary in $\mathcal{S}(C)$. In $\mathcal{C}(F)$, if $s(C)$ is related to $s(D)$ by replacing one zero entry of $C$ with a $1$ or $-1$, by Lemmas \ref{lem:faceproperties} and \ref{lem:countzeros}  that this is equivalent to $C \leq D$ and $dim(C)+1=dim(D)$, which is equivalent to $C$ being in the (mod two) cellular boundary of $D$. 
\end{proof}

Lastly, we prove that the process of computing $\mathcal{C}(F)$ can be done iteratively through layers, beginning with the first layer:


\begin{lemma} \label{lem:firstlayervertices}
	Let $F$ be a supertransversal, generic neural network.
	
	The $0$-cells of $\mathcal{C}(F_{(1)})$ are given by the solutions to $$\{W_{\alpha} x = b_\alpha: \alpha \subset [n_1] \;\& \; |\alpha|=n_0  \}$$ where $W$ is the weight matrix of the network and $\alpha$ denotes a subset of the $n_1$ vertices. 
	
	A vertex $v$ obtained by solving $W_\alpha x = b_\alpha$ satisfies $s_{i}(v)=0$ iff $i \in \alpha$.  
\end{lemma}

\begin{proof}
	These are the vertices of a generic, affine hyperplane arrangement. 
\end{proof}
 
 In order to compute the vertices of $\mathcal{C}(F)$ corresponding to bent hyperplanes from further layers, we loop through regions $C$ of $\mathcal{C}(F_{(k-1)})$ and solve systems of linear equations arising from $n_0$ bent hyperplanes on that region, at least one of which corresponds to a new bent hyperplane $F_{kj}$. The following lemma guarantees that if we select these combinations of $F_{ij}$ corresponding to earlier layers from only those which intersect to form cells on the boundary of $\mathcal{C}$, we are guaranteed to obtain all new vertices in $\mathcal{C}(F_{(k)})$. Furthermore, once such an intersection $x$ is found with new bent hyperplanes we may determine whether the intersection belongs to the polyhedral complex by evaluating $F_{ij}(x)$ at only the bent hyperplanes we did not intersect. Thus, we do not have to determine whether $\sgn(F_{ij}(x))=0$, removing a source of floating point error. 


\begin{lemma}\label{lem:laterlayervertices}
	Let $F$ be a generic, supertransversal neural network with at least $n_0$ hidden units in its first layer. 

	If $C$ is a cell of $\mathcal{C}(F_{k-1} \circ ... \circ F_1)$, then $F_{ij}(C)$ is affine for all $i \leq k$. Call the corresponding affine map $A_{ij}: \mathbb{R}^{n_0}\to \mathbb{R}$. Then, 
	
	\begin{enumerate}
	\item	All $0$-cells of $\mathcal{C}(F_k \circ ... \circ F_1)$  which are contained in the closure of $C$ and which are not already in $\mathcal{C}(F_{k-1}\circ ... \circ F_1)$ are the solution to a system of $n_0$ affine equations, of which $1 \leq \ell \leq n_0$ are of the form:
		
		$$ A_{km}(x)=0  $$
		
		and $0 \leq n_0 - \ell \leq n_0 - 1$ equations are of the form: 
		
		$$A_{ij}(x) = 0 ; i < k$$ 
		
		Here, the $A_{ij}$ of the $n_0-\ell $ equations from earlier layers are selected such that there exists a vertex of $C$ in the intersection of the corresponding bent hyperplanes. In other words, the remaining $n_0-\ell$ equations describe the affine span of a face of $C$.
		
	\item 	A solution to the system of equations described in (1) corresponds to a $0$-cell of $\mathcal{C}(F_k\circ ... \circ F_1)$ contained in the closure of $C$ if and only if, for all \textbf{remaining} $(i,j)$ pairs with $i \leq k-1$, we have that $s_{ij}(v)=s_{ij}(C)$.
	\end{enumerate}
\end{lemma}

\begin{proof}
	
	For statement (1), suppose that $v$ is in the closure of $C$, where $C$ is a cell of $\mathcal{C}(F_{k-1} \circ ... \circ F_1)$, and $v$ is a vertex of $\mathcal{C}(F_k \circ ... \circ F_1)$. By Theorem \ref*{thm:cubicalcomplex}, $v$ is the solution to $F_{ij}(x)=0$ for exactly $n_0$ node maps. Since $F_{ij}|_{C}=A_{ij}$, then $A_{ij}(v)=0$ for those $n_0$ node maps. If $i<k$ for all of these node maps $F_{ij}$, then in fact $v$ is a $0$-cell of $\mathcal{C}(F_{k-1}\circ...\circ F_1)$. So if $v$ is a vertex of $\mathcal{C}(F_{k}\circ ... \circ F_1)$ and not a vertex of $\mathcal{C}(F_{k-1} \circ ... \circ F_1)$, at least one of the $F_{ij}$ must be a node map with $i=k$. Thus, any vertex of $\mathcal{C}(F_{k}\circ ... \circ F_1)$ which is contained in the closure of $C$ must be a solution to a system of equations of this form. For any solution of this form to be nonempty when intersecting with the closure of $C$, the the $A_{ij}$ corresponding to this system of equations must satisfy the condition that $C \cap \bigcap \{x: A_{ij}(x)=0\}$ is nonempty. Since none of the $A_{ij}$ from earlier layers intersect the interior of $C$, the intersection of the $A_{ij}$ are describing the linear span of a face of $C$, which must contain a vertex of $C$. 
	
	For statement (2), of course if $v$ is a solution to the system of equations described in (1) and also is in the closure of $C$, then by Lemma \ref*{lem:faceproperties}, $s_{ij}(v) = s_{ij}(C)$ when $i\leq k-1$, except for where $s_{ij}(v)=0$, which by Theorem \ref{thm:cubicalcomplex} occurs for precisely the bent hyperplanes which were intersected to obtain $s_{ij}$. 
	
	\begin{figure}[h]
		\centering 
		\includegraphics[width=0.4\linewidth]{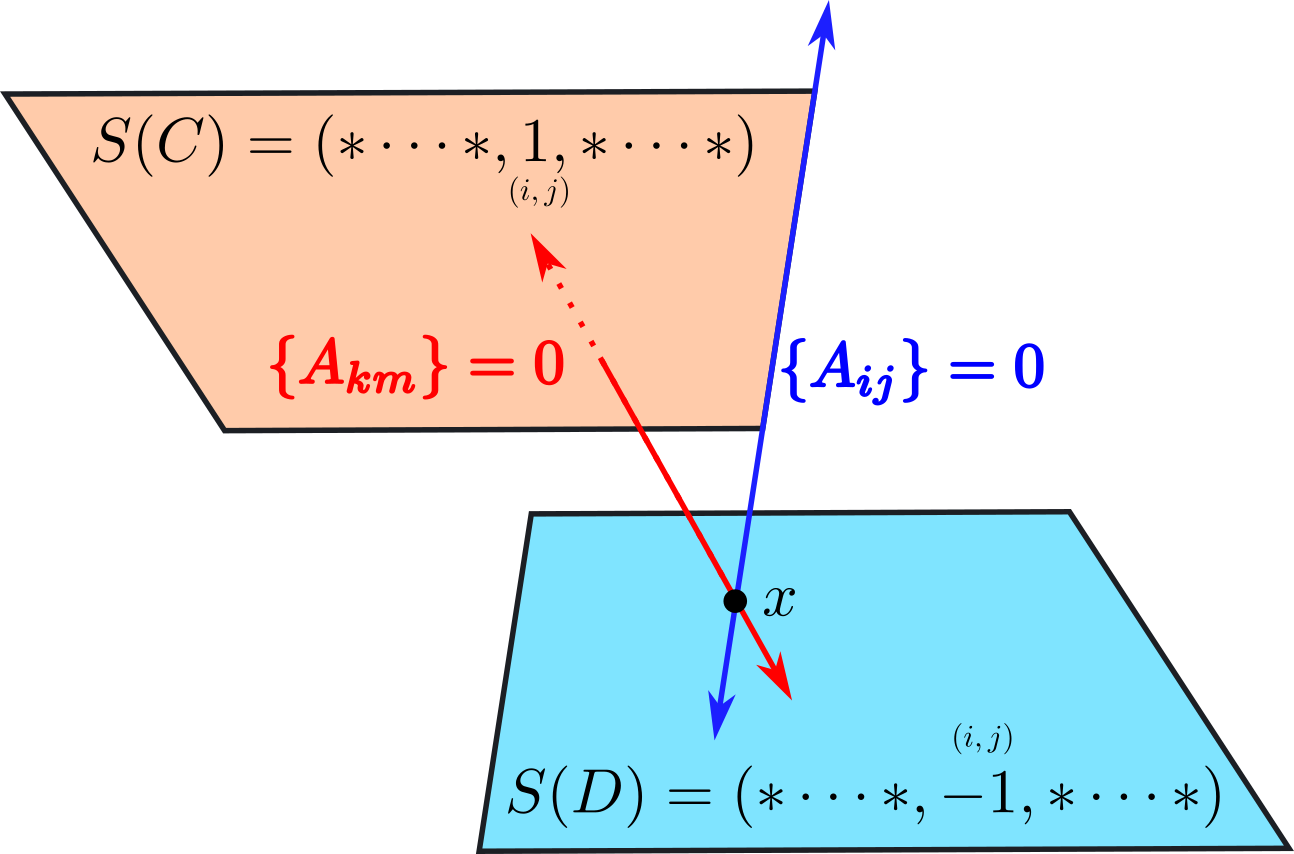}
		\caption{When determining if a solution $x$ to the system of equations in Lemma \ref{lem:laterlayervertices} is a vertex of $\mathcal{C}(F_{(k)})$, we look at its sign sequence. However, its sign relative to $\{A_{ij}\}$ is numerically unstable. We would like to guarantee it belongs to the closure of $C$  by evaluating the node maps which do not include $A_{ij}$. A concern is that it is contained in a different cell $D$ with identical signs to $C$ in $\mathcal{C}(F_{(k-1)})$ except possibly in the locations of the $A_{ij}$ which we intersected, which would make this task impossible. The argument in part (2) shows this does not occur, and the situation pictured above is impossible. } 
				\label{fig:contradictionfigure}
	\end{figure}
	
	In the other direction, if $x$ is a solution to the above system of equations but is not a $0$-cell of $\mathcal{C}(F)$, it must not be contained in the closure of $C$.  Then $x$ is contained in the interior of some other cell of $\mathcal{C}(F_{k-1}\circ ... \circ F_1)$, call it $D$, such that $D$ is not a face of $C$ (Figure \ref{fig:contradictionfigure}). If there is some bent hyperplane corresponding to one of the remaining $(i,j)$  pairs such that $s_{ij}(D)\neq s_{ij}(C)$ then we are done. Otherwise we will see a contradiction. If $S(D)=S(C)$ except at $(i,j)$ pairs corresponding to some of the $A_{ij}$, then by our selection of equations earlier, there is a face $E$ of $C$ which has the sign sequence equal to zero at these coordinates (contained in the intersection of the solution of $A_{ij}x=0$). If $E=D$, then $D$ is a face of $C$ and we have a contradiction. The only other option is that $E$ is a proper face of $D$ by Lemma \ref{lem:faceproperties}. The intersection of the solutions to $A_{ij}=0$ contains the affine span of $E$, so the intersection of these with the closure of $D$ is contained in a proper face of $D$, and so $x$, an element of this intersection, cannot be in the interior of $D$. This contradicts our assumption that $x$ is in the interior of $D$.
	
	This shows that if $x$ is a solution to the above system of equations but is not a $0$-cell of $\mathcal{C}(F)$, then there exists some $(i,j)$ pair with $i\leq k-1$ such that $s_{ij}(x)\neq s_{ij}(C)$ and which does not correspond to the hyperplanes which were intersected.

\end{proof}

\section{License Information}
\label{a:license}

PyTorch \cite{pytorch} is under a Modified BSD license, permitting use in other projects and requiring its licensing information repackaged when its source code is redistributed. We do not redistribute its source code in our work. 

Sage \cite{sage} is licensed under the GNU General Public License (GPL). It is free to use and distribute. We do not redistribute its source code in our work, but it is necessary to run the decision boundary topology computations.

\end{document}